\documentclass{article}

\usepackage[preprint,nonatbib]{neurips_2021}

\usepackage[utf8]{inputenc} \usepackage[T1]{fontenc}    \usepackage{hyperref}       \usepackage{url}            \usepackage{booktabs}       \usepackage{amsfonts}       \usepackage{nicefrac}

\usepackage[activate={true,nocompatibility},final,kerning=true,stretch=0, shrink=50]{microtype}
\usepackage{xcolor}         

\usepackage[utf8]{inputenc} \usepackage[T1]{fontenc}    \usepackage{tabularx}
\usepackage{hyperref}       \usepackage{url}            \usepackage{placeins}
\usepackage{booktabs}       \usepackage{amsfonts}       \usepackage{nicefrac}       \usepackage{microtype}      \usepackage[numbers]{natbib}
\usepackage{proba}
\usepackage{wrapfig}
\usepackage{amssymb}\usepackage{pifont}\usepackage{subcaption}
\usepackage{rotating}
\usepackage{etoolbox}
\usepackage{xparse}
\usepackage{grffile}
\usepackage{amsmath}
\usepackage{xspace}
\usepackage{enumitem}
\usepackage{euscript}
\usepackage{mathtools}
\usepackage{tikz}
\usepackage{enumitem}
\usepackage{mathtools}
\usepackage{tablefootnote}
\usepackage[flushleft]{threeparttable}
\usepackage{multirow}
\usepackage[title]{appendix}
\usepackage{arydshln}
\usepackage{array, booktabs}\usepackage{comment}
\usepackage[utf8]{inputenc}
\usepackage{url}
\usepackage{graphicx}
\usepackage{amssymb}
\usepackage{dsfont}
\usepackage{amsmath,amsthm}
\usepackage{mathtools}
\usepackage{booktabs}
\usepackage{balance}
\usepackage{xspace}
\usepackage{multirow}
\usepackage{multicol}
\usepackage{xcolor}
\usepackage{comment}
\usepackage{nameref}
\usepackage{pdfpages}
\usepackage{enumitem}
\usepackage{tikz}

\usepackage[normalem]{ulem}

\newcommand*\dbar[1]{\overline{\overline{\lower0.2ex\hbox{$#1$}}}}

\usetikzlibrary{shapes.misc}
\colorlet{darkgreen}{green!50!black}

\NewDocumentCommand{\Exp}{d() r[]}{\ensuremath{
		\mathds{E}
		\IfValueT{#1}{_{#1}}
		{\left[#2\right]}
}}

\newtheorem{theorem}{Theorem}

\newtheorem{conjecture}{Conjecture}

\newtheorem{definition}{Definition}

\newtheorem{lemma}{Lemma}

\newtheorem{proposition}{Proposition}

\newcommand{\eat}[1]{}

\newcommand{\cD}{\mathcal{D}}
\newcommand{\cT}{\mathcal{T}}

\newcommand{\cY}{\mathcal{Y}}
\newcommand{\cO}{\mathcal{O}}
\newcommand{\cI}{\mathcal{I}}

\newcommand{\cS}{\mathcal{S}}

\newcommand{\real}{\mathbb{R}}

\newcommand{\bW}{{\bf W}}

\definecolor{pastelblue}{RGB}{76,113,175}
\definecolor{pastelgreen}{RGB}{84,167,104}
\definecolor{pastelred}{RGB}{196,78,82}
\definecolor{pastelgrey}{RGB}{230,230,230}
\definecolor{pastelbeige}{RGB}{243,236,221}
\definecolor{pastelpurple}{RGB}{154,139,192}
\definecolor{mysalmon}{RGB}{250, 128, 114}
\definecolor{myblue}{RGB}{60,105,210}
\definecolor{mygreen}{RGB}{60,179,113}

\usepackage{amsmath,amsfonts,bm}

\def\eqref#1{equation~\ref{#1}}

\def\ceil#1{\lceil #1 \rceil}

\def\1{\bm{1}}

\DeclareMathAlphabet{\mathsfit}{\encodingdefault}{\sfdefault}{m}{sl}
\SetMathAlphabet{\mathsfit}{bold}{\encodingdefault}{\sfdefault}{bx}{n}

\def\cD{{\mathcal{D}}}

\def\cG{{\mathcal{G}}}

\def\cI{{\mathcal{I}}}

\def\cO{{\mathcal{O}}}

\def\cR{{\mathcal{R}}}
\def\cS{{\mathcal{S}}}
\def\cT{{\mathcal{T}}}

\def\cY{{\mathcal{Y}}}

\def\sR{{\mathbb{R}}}

\newcommand{\modelclass}{Reconstruction Neural Network}
\newcommand{\kmodel}{$k$-Reconstruction Neural Network}
\newcommand{\recmodel}{Full Reconstruction Neural Network}
\newcommand{\gnn}{$k$-Reconstruction GNN}

\newcommand*{\lms}{\{\mskip-5mu\{}
\newcommand*{\rms}{\}\mskip-5mu\}}

\makeatletter
\def\thmt@refnamewithcomma #1#2#3,#4,#5\@nil{\@xa\def\csname\thmt@envname #1utorefname\endcsname{#3}\ifcsname #2refname\endcsname
	\csname #2refname\expandafter\endcsname\expandafter{\thmt@envname}{#3}{#4}\fi
}
\makeatother
\usepackage[capitalise,noabbrev]{cleveref}

\newtheorem{observation}{Observation}

\Crefname{conjecture}{Conjecture}{Conjectures}
\Crefname{observation}{Observation}{Observations}

\usepackage{xfrac}

\definecolor{pastelblue}{RGB}{76,113,175}
\definecolor{pastelgreen}{RGB}{144,238,144}
\definecolor{pastelred}{RGB}{196,78,82}
\definecolor{pastelgrey}{RGB}{230,230,230}
\definecolor{pastelbeige}{RGB}{243,236,221}
\definecolor{pastelpurple}{RGB}{154,139,192}
\definecolor{mysalmon}{RGB}{250, 128, 114}
\definecolor{myblue}{RGB}{60,105,210}
\definecolor{mygreen}{RGB}{60,179,113}

\usepackage{bbding}
\definecolor{dkgreen}{rgb}{0,0.6,0}
\definecolor{dkred}{rgb}{0.5,0,0}

\usepackage{floatrow}
\newfloatcommand{capbtabbox}{table}[][\FBwidth]

\newcommand{\concat}[1]{\textsc{Concat}(#1)}
 \newcommand{\xhdr}[1]{{\noindent\bfseries #1}}

\title{Reconstruction for Powerful Graph Representations}

\author{Leonardo Cotta \\
  Purdue University\\
  \texttt{cotta@purdue.edu} \\
  \And
  Christopher Morris \\
  Mila -- Quebec AI Institute, McGill University \\
  \texttt{chris@christophermorris.info} \\
  \And
  Bruno Ribeiro \\
  Purdue University\\
  \texttt{ribeiro@cs.purdue.edu} \\
}

\begin{document}

\maketitle

\begin{abstract}
Graph neural networks (GNNs) have limited expressive power, failing to represent many graph classes correctly. While more expressive graph representation learning (GRL) alternatives can distinguish some of these classes, they are significantly harder to implement, may not scale well, and have not been shown to outperform well-tuned GNNs in real-world tasks. Thus, devising simple, scalable, and expressive GRL architectures that also achieve real-world improvements remains an open challenge. In this work, we show the extent to which graph reconstruction---reconstructing a graph from its subgraphs---can mitigate the theoretical and practical problems currently faced by GRL architectures. First, we leverage graph reconstruction to build two new classes of expressive graph representations. Secondly, we show how graph reconstruction boosts the expressive power of any GNN architecture while being a (provably) powerful inductive bias for invariances to vertex removals. Empirically,  we show how reconstruction can boost GNN's expressive power---while maintaining its invariance to permutations of the vertices---by solving seven graph property tasks not solvable by the original GNN. Further, we demonstrate how it boosts state-of-the-art GNN's performance across nine real-world benchmark datasets.
\end{abstract}
\vspace{-0.5em}
\section{Introduction}

Supervised machine learning for graph-structured data, i.e., graph classification and regression, is ubiquitous across application domains ranging from chemistry and bioinformatics~\citep{Barabasi2004,Sto+2020} to image~\citep{Sim+2017}, and social network analysis~\citep{Eas+2010}. Consequently, machine learning on graphs is an active research area with numerous proposed approaches---notably GNNs~\citep{Cha+2020,Gil+2017,Gro+2020} being the most representative case of GRL methods.

Arguably, GRL's most interesting results arise from a cross-over between graph theory and representation learning. For instance, the representational limits of GNNs are upper-bounded by a simple heuristic for the graph isomorphism problem~\citep{Mor+2019,Xu+2018b}, the \emph{$1$-dimensional Weisfeiler-Leman algorithm} ($1$-WL)~\citep{Gro2017,Mor+2021,Wei+1976,Wei+1968}, which might miss crucial structural information in the data~\citep{Arv+2015}. Further works show how GNNs cannot approximate graph properties such as diameter, radius, girth, and subgraph counts~\citep{Chen+2020,Gar+2020}, inspiring architectures~\cite{azizian2020characterizing,Mar+2019,Mor+2019,Morris2020b} based on the more powerful \emph{$\kappa$-dimensional Weisfeiler-Leman algorithm} ($\kappa$-WL)~\cite{Gro2017}.\footnote{We opt for using $\kappa$ instead of $k$, i.e., $\kappa$-WL instead of $k$-WL, to not confuse the reader with the  hyperparameter $k$ of our models.}
On the other hand, despite the limited expressiveness of GNNs, they still can overfit the training data, offering limited generalization performance~\citep{Xu+2018b}. Hence, devising GRL architectures that are simultaneously sufficiently expressive and avoid overfitting remains an open problem.

An under-explored connection between graph theory and GRL is graph reconstruction, which studies graphs and graph properties uniquely determined by their subgraphs. In this direction, both the pioneering work of \citet{shawe1993symmetries} and the more recent work of \citet{bouritsas2020improving}, show that assuming the reconstruction conjecture (see \Cref{conj:recon}) holds, their models are most-expressive representations (universal approximators) of graphs. Unfortunately, \citeauthor{shawe1993symmetries}'s computational graph grows exponentially with the number of vertices, and \citeauthor{bouritsas2020improving}'s full representation power requires performing multiple graph isomorphism tests on potentially large graphs (with $n-1$ vertices). Moreover, these methods were not inspired by the more general subject of graph reconstruction; instead, they rely on the reconstruction conjecture to prove their architecture's expressive powers.

\xhdr{Contributions.}
In this work, we directly connect graph reconstruction to GRL. We first show how the \emph{$k$-reconstruction of graphs}---reconstruction from induced $k$-vertex subgraphs---induces a natural class of expressive GRL architectures for supervised learning with graphs, denoted \emph{\kmodel{s}}. We then show how several existing works have their expressive power limited by $k$-reconstruction. Further, we show how the reconstruction conjecture's insights lead to a provably most expressive representation of graphs. Unlike \citet{shawe1993symmetries} and \citet{bouritsas2020improving}, which, for graph tasks, require fixed-size unattributed graphs and multiple (large) graph isomorphism tests, respectively, our method represents bounded-size graphs with vertex attributes and does not rely on isomorphism tests.

To make our models scalable, we propose \emph{\gnn{s}}, a general tool for boosting the expressive power and performance of GNNs with graph reconstruction. Theoretically, we characterize their expressive power showing that \gnn{s} can distinguish graph classes that the $1$-WL and $2$-WL cannot, such as cycle graphs and strongly regular graphs, respectively. Further, to explain gains in real-world tasks, we show how reconstruction can act as a lower-variance risk estimator when the graph-generating distribution is invariant to vertex removals. Empirically, we show that reconstruction enhances GNNs' expressive power, making them solve multiple synthetic graph property tasks in the literature not solvable by the original GNN. On real-world datasets, we show that the increase in expressive power coupled with the lower-variance risk estimator boosts GNNs' performance up to 25\%.
Our combined theoretical and empirical results make another important connection between graph theory and GRL.

\subsection{Related work} 
We review related work from GNNs, their limitations, data augmentation, and the reconstruction conjecture in the following. See~\cref{ext_rw} for a more detailed discussion.

\xhdr{GNNs.} Notable instances of this architecture include, e.g.,~\citep{Duv+2015,Ham+2017,Vel+2018}, and the spectral approaches proposed in, e.g.,~\citep{Bru+2014,Def+2015,Kip+2017,Mon+2017}---all of which descend from early work in~\citep{bas+1997,Kir+1995,Mer+2005,mic+2009,mic+2005,Sca+2009,Spe+1997}. Aligned with the field's recent rise in popularity, there exists a plethora of surveys on recent advances in GNN methods. Some of the most recent ones include~\citep{Cha+2020,Wu+2018,Zho+2018}.

\xhdr{Limits of GNNs.} Recently, connections to Weisfeiler-Leman type algorithms have been shown~\citep{Bar+2020,Che+2019,Gee+2020a,Gee+2020b,Mae+2019,Mar+2019,Mor+2019,Morris2020b,Xu+2018b}. Specifically, the authors of~\citep{Mor+2019,Xu+2018b} show how the 1-WL limits the expressive power of any possible GNN architecture. \citet{Mor+2019} introduce \emph{$\kappa$-dimensional GNNs} which rely on a more expressive message-passing scheme between subgraphs of cardinality $\kappa$. Later, this was refined in~\citep{azizian2020characterizing, Mar+2019} and in \citep{Mor+2019b} by deriving models equivalent to the more powerful $ \kappa $-dimensional Weisfeiler-Leman algorithm. \citet{Che+2019} connect the theory of universal approximation of permutation-invariant functions and graph isomorphism testing, further introducing a variation of the $2$-WL. Recently, a large body of work propose enhancements to GNNs, e.g., see~\cite{Alb+2019,beaini2020directional,bodnar2021weisfeiler,bouritsas2020improving,murphy2019relational,Vig+2020,you2021identity}, making them more powerful than the $1$-WL; see~\cref{ext_rw} for a in-depth discussion. For clarity, throughout this work, we will use the term GNNs to denote the class of message-passing architectures limited by the $1$-WL algorithm, where the class of distinguishable graphs is well understood~\citep{Arv+2015}.

\xhdr{Data augmentation, generalization and subgraph-based inductive biases.}
There exist few works proposing data augmentation for GNNs for graph classification. \citet{Kon+2020} introduces a simple feature perturbation framework to achieve this,  while~\citet{Ron+2020,Fen+2020} focus on vertex-level tasks. \citet{Gar+2020} study the generalization abilities of GNNs showing bounds on the Rademacher complexity, while \citet{Lia+2020} offer a refined analysis within the PAC-Bayes framework.
Recently, \citet{bouritsas2020improving} proposed to use subgraph counts as vertex and edge features in GNNs. Although the authors show an increase in expressiveness, the extent, e.g., which graph classes their model can distinguish, is still mostly unclear. Moreover,~\citet{Yeh+2020} investigate GNNs' ability to generalize to larger graphs. Concurrently,~\citet{bevilacqua2021size} show how subgraph densities can be used to build size-invariant graph representations. However, the performance of such models in in-distribution tasks, their expressiveness, and scalability remain unclear.
Finally,~\citet{yuan2021explainability} show how GNNs' decisions can be explained by (often large) subgraphs, further motivating our use of graph reconstruction as a powerful inductive bias for GRL.

\xhdr{Reconstruction conjecture.} The reconstruction conjecture is a longstanding open problem in graph theory, which has been solved in many particular settings. Such results come in two flavors. Either proving that graphs from a specific class are reconstructible or determining which graph functions are reconstructible. Known results of the former are, for instance, that regular graphs, disconnected graphs, and trees are reconstructible~\citep{bondymanual,kellytrees}. In particular, we highlight that outerplanar graphs, which account for most molecule graphs, are known to be reconstructible~\citep{Gil+1974}. 
For a comprehensive review of graph reconstruction results, see \citet{bondymanual}.

\section{Preliminaries}\label{sec:prelims}

\begin{figure*}[t!]
\centering
        \includegraphics[scale=0.7]{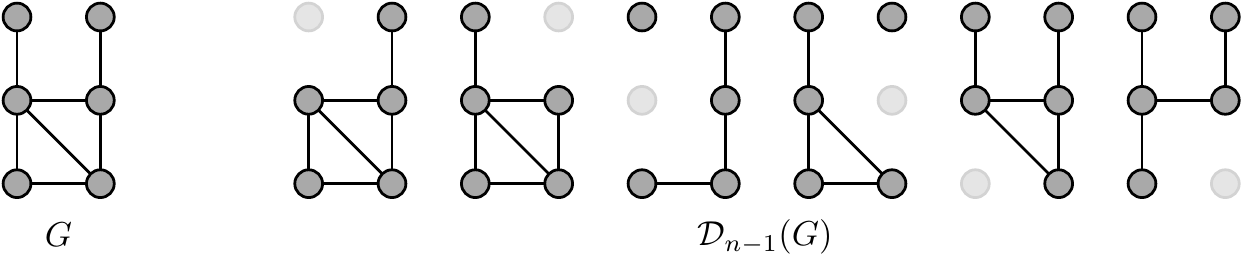}\caption{A graph $G$ and its deck $\mathcal{D}_{n-1}(G)$, faded out vertices are not part of each card in the deck.}
    \label{fig:deck}
\vspace{-1em}
\end{figure*}

Here, we introduce notation and give an overview of the main results in graph reconstruction theory~\citep{bondymanual,godsil1993algebraic}, including the reconstruction conjecture~\citep{ulam}, which forms the basis of the models in this work.

\xhdr{Notation and definitions.} As usual, let $[n] = \{ 1, \dotsc, n \} \subset \N$ for $n \geq 1$, and let $\{\!\!\{ \dots\}\!\!\}$ denote a multiset. In an abuse of notation, for a set $X$ with $x$ in $X$, we denote by $X - x$ the set $X \setminus \{x\}$. We also assume elementary definitions from graph theory, such as graphs, directed graphs, vertices, edges, neighbors, trees, isomorphism, et cetera; see~\cref{notation}. The vertex and the edge set of a graph $G$ are denoted by $V(G)$ and $E(G)$, respectively. The \emph{size} of a graph $G$ is equal to its number of vertices. Unless indicated otherwise, we use $n:=|V(G)|$. If not otherwise stated, we assume that vertices and edges are annoted with \emph{attributes}, i.e., real-valued vectors.

We denote the set of all finite and simple graphs by $\cG$. The subset of $\cG$ without edge attributes (or edge directions) is denoted $\mathfrak{G} \subset \cG$. We write $G \simeq H$ if the graphs $G$ and $H$ are isomorphic. Further, we denote the isomorphism type, i.e., the equivalence class of the isomorphism relation, of a graph $G$ as $\cI(G)$. Let $S \subseteq V(G)$, then $G[S]$ is the induced subgraph with edge set $E(G)[S] = \{ S^2 \cap E(G) \}$. We will refer to induced subgraphs simply as subgraphs in this work. 

Let $\mathfrak{R}$ be a family of graph representations, such that for $d\geq 1$, $r$ in $\mathfrak{R}$, $r \colon \cG \to \sR^d$, assigns a $d$-dimensional representation vector $r(G)$ for a graph $G$ in $\cG$. We say $\mathfrak{R}$ can \emph{distinguish} a graph $G$ if there exists $r$ in $\mathfrak{R}$ that assigns a unique representation to the isomorphism type of $G$, i.e., $r(G) = r(H)$ if and only if $ G \simeq H $. 
Further, we say $\mathfrak{R}$ distinguishes a pair of non-isomorphic graphs $G$ and $H$ if there exists some $r$ in $\mathfrak{R}$ such that $r(G) \neq r(H)$. Moreover, we write $\mathfrak{R}_1 \preceq \mathfrak{R}_2$ if $\mathfrak{R}_2$ distinguishes between all graphs $\mathfrak{R}_1$ does, and $\mathfrak{R}_1 \equiv \mathfrak{R}_2$ if both directions hold. The corresponding strict relation is denoted by $\prec$. Finally, we say $\mathfrak{R}$ is a \emph{most-expressive representation} of a class of graphs if it distinguishes all non-isomorphic graphs in that class. 

\xhdr{Graph reconstruction.} Intuitively, the reconstruction conjecture states that an undirected edge-unattributed graph can be fully recovered up to its isomorphism type given the multiset of its vertex-deleted subgraphs' isomorphism types. This multiset of subgraphs is usually referred to as the \emph{deck} of the graph, see \Cref{fig:deck} for an illustration. Formally, for a graph $G$, we define its deck as $
    \cD_{n-1}(G) = \lms \cI(G[V(G) - v]) \colon v \in V(G) \rms$.
We often call an element in $\cD_{n-1}(G)$ a \emph{card}. We define the graph reconstruction problem as follows.
\begin{definition}\label{def:recon}
    Let $G$ and $H$ be graphs, then $H$ is a \emph{reconstruction} of $G$  if $H$ and $G$ have the same deck, denoted $H \sim G$. A graph $G$ is \emph{reconstructible} if every reconstruction of $G$ is isomorphic to $G$, i.e., $H \sim G$ implies $H \simeq G$. 
\end{definition}
\vspace{-0.5em}
Similarly, we define \emph{function reconstruction}, which relates functions that map two graphs to the same value if they have the same deck.

\begin{definition}\label{def:func-recon}
Let  $f\colon \cG \to \cY$ be a function, then $f$ is \emph{reconstructible} if $f(G)=f(H)$ for all graphs in $\{(H,G) \in \cG^2 \colon H \sim G \}$, i.e., $G \sim H$ implies $f(G) = f(H) $.
\end{definition}
\vspace{-0.5em}

We can now state the reconstruction conjecture, which in short says that every $G$ in $\mathfrak{G}$ with $ |V| \geq 3$ is reconstructible.

\begin{conjecture}[\citet{kelly,ulam}]\label{conj:recon}
	Let $H$ and $G$ in $\mathfrak{G}$ be two finite, undirected, simple graphs with at least three vertices. If $H$ is a reconstruction of $G$, then $H$ and $G$ are isomorphic.
\end{conjecture}
\vspace{-0.5em}
We note here that the reconstruction conjecture does not hold for directed graphs, hypergraphs, and infinite graphs~\citep{bondymanual,stockmeyer1977falsity,stockmeyer1981census}. In particular, edge directions can be seen as edge attributes. Thus, the reconstruction conjecture does not hold for the class $\cG$. In contrast, the conjecture has been proved for practical-relevant graph classes, such as disconnected graphs, regular graphs, trees, and outerplanar graphs~\citep{bondymanual}. Further, computational searches show that graphs with up to 11 vertices are reconstructible~\citep{mckay1997small}. Finally, many graph properties are known to be reconstructible, such as every size subgraph count, degree sequence, number of edges, and the characteristic polynomial~\citep{bondymanual}.

\xhdr{Graph \texorpdfstring{$k$}{}-reconstruction.} \citet{kellytrees} generalized graph reconstruction, considering the multiset of subgraphs of size $k$ instead of $n-1$, which we denote $\cD_{k}(G) = \lms \cI(H) \colon H \in \cS^{(k)}(G) \rms$,
where $\cS^{(k)}$ is the set of all $\binom{n}{k}$ $k$-size subsets of $V$. We often call an element in $\cD_{k}(G)$ a $k$-\emph{card}. From the $k$-deck definition, it is easy to extend the concept of graph and function reconstruction, cf.~\Cref{def:recon,def:func-recon}, to \emph{graph} and \emph{function $k$-reconstruction}.

\begin{definition}\label{def:krecon}
    Let $G$ and $H$ be graphs, then $H$ is a \emph{$k$-reconstruction} of $G$  if $H$ and $G$ have the same $k$-deck, denoted $H \sim_k G$. A graph $G$ is \emph{$k$-reconstructible} if every $k$-reconstruction of $G$ is isomorphic to $G$, i.e., $H \sim_k G$ implies $H \simeq G$. 
\end{definition}
\vspace{-0.5em}
Accordingly, we define $k$-\emph{function reconstruction} as follows.
\begin{definition}\label{def:func-krecon}
Let  $f\colon \cG \to \cY$ be a function, then $f$ is $k$-reconstructible if $f(G)=f(H)$ for all graphs in $\{(H,G) \in \cG^2 \colon H \sim_k G \}$, i.e., $G \sim_k H$ implies $f(G) = f(H) $.
\end{definition}
\vspace{-0.5em}

Results for $k$-reconstruction usually state the least $k$ as a function of $n$ such that all graphs $G$ in $\cG$ (or some subset) are $k$-reconstructible~\citep{nydl2001graph}. There exist extensive partial results in this direction, mostly describing $k$-reconstructibility (as a function of $n$) for a particular family of graphs, such as trees, disconnected graphs, complete multipartite graphs, and paths, see \citep{nydl2001graph,kostochka2019reconstruction}. More concretely, \citet{nydl1981finite,spinoza2019reconstruction} showed graphs with $2k$ vertices that are not $k$-reconstructible. In practice, these results imply that for some fixed $k$ there will be graphs with not many more vertices than $k$ that are not $k$-reconstructible. 
Further, $k$-reconstructible graph functions such as degree sequence and connectedness have been studied in~\citep{manvel1974some,spinoza2019reconstruction} depending on the size of $k$. In \Cref{supp:rec}, we discuss further such results.

\section{\modelclass{}s}\label{sec:recnets}

Building on the previous section, we propose two neural architectures based on graph $k$-reconstruction and graph reconstruction. First, we look at \emph{\kmodel{s}}, the most natural way to use graph $k$-reconstruction. Secondly, we look at \emph{\recmodel{s}}, where we leverage the Reconstruction Conjecture to build a most-expressive representation for the class of graphs of bounded size and unattributed edges.
 
\xhdr{\texorpdfstring{\kmodel{s}}{}.} Intuitively, the key idea of \kmodel{s} is that of learning a joint representation based on subgraphs induced by $k$ vertices. Formally, let $ f_{\bW} \colon \cup^{\infty}_{m=1}  \real^{m \times d} \to \real^{t}$ be a (row-wise) permutation-invariant function and $\cG_k = \{ G \in \cG \colon |V(G)| = k\}$ be the set of graphs with exactly $k$ vertices. Further, let  $h^{(k)} \colon \cG_k \to \real^{1 \times d}$ be a graph representation function such that two graphs $G$ and $H$ on $k$ vertices  are mapped to the same vectorial representation if and only if they are isomorphic, i.e., $h^{(k)}(G)=h^{(k)}(H) \iff G \simeq H$ for all $G$ and $H$ in $\cG_k$. We define \kmodel{s} over $\cG$ as a function with parameters $\bW$ in the form
	\begin{equation*}
	r^{(k)}_{\bW}(G) = f_{\bW} \left(\concat{\lms h^{(k)}(G[S]) \colon S \in \cS^{(k)} \rms }\right)\!,
	\end{equation*}
where $\cS^{(k)}$ is the set of all $k$-size subsets of $V(G)$ for some $3 \leq k \leq n$, and $\textsc{Concat}$ denotes row-wise concatenation of a multi-set of vectors in some arbitrary order. Note that $h^{(k)}$ might also be a function with learnable parameters. In that case, we require it to be most-expressive for $\cG_k$. The following results characterize the expressive power of the above architecture.

\begin{proposition}\label{prop:kmodel}
Let $f_{\bW}$ be a universal approximator of multisets~\citep{zaheer2017deep,wagstaff19,murphy2019janossy}. Then, $r^{(k)}_{\bW}$ can approximate a function if and only if the function is $k$-reconstructible. 
\end{proposition}
Moreover, we can observe the following.
\begin{observation}[\citet{nydl2001graph, kostochka2019reconstruction}]\label{obs:kdeck}
For any graph $G$ in $\cG$, its $k$-deck $\cD_k(G)$ determines its $(k-1)$-deck $\cD_{k-1}(G)$. \end{observation}

From \Cref{obs:kdeck}, we can derive a hierarchy in the expressive power of \kmodel{s} with respect to the subgraph size $k$. That is, 
$r^{(3)}_\bW \preceq r^{(4)}_\bW \preceq \dots \preceq r^{(n-2)}_\bW  \preceq r^{(n-1)}_\bW .$

In \Cref{supp:kmodel}, we show how many existing architectures have their expressive power limited by $k$-reconstruction. 
We also refer to \Cref{supp:kmodel} for the proofs, a discussion on the model's computational complexity, approximation methods, and relation to existing work.

\xhdr{\texorpdfstring{\recmodel{s}}{}.} Here, we propose a recursive scheme based on the reconstruction conjecture to build a most-expressive representation for graphs. Intuitively, \recmodel{s} recursively compute subgraph representations based on smaller subgraph representations. Formally, let $\mathfrak{G}_{\leq n^*}^\dagger := \{ G \in \mathfrak{G} \colon |V(G)| \leq n^*\}$ be the class of undirected graphs with unattributed edges and maximum size $n^*$. Further, let $f^{(k)}_{\bW} \colon \cup^{\infty}_{m=1} \real^{m \times d} \to \real^{t}$ be a (row-wise) permutation invariant function and let  $h_{\{i,j\}}$ be a most-expressive representation of the two-vertex subgraph induced by vertices $i$ and $j$. We can now define the representation $r(G[V(G)])$ of a graph $G$ in $\mathfrak{G}_{\leq n^*}^\dagger$ in a recursive fashion as 
\begin{equation*}
  r(G[S]) =
\begin{cases}
   f^{(|S|)}_{\bW} \left( \concat{\lms r{(G[S - v ])} \colon v \in S \rms } \right),  \text{ if } 3 \leq |S| \leq n \\
    h_S(G[S]),  \text{\hspace{135pt} if } |S|=2.
\end{cases}
\end{equation*}

Again, $\textsc{\concat{}}$ is row-wise concatenation in some arbitrary order. Note that in practice, it is easier to build the subgraph representations in a bottom-up fashion. First, use two-vertex subgraph representations to compute all three-vertex subgraph representations. Then, perform this inductively until we arrive at a single whole-graph representation. In \Cref{supp:recmodel}, we prove the expressive power of \recmodel{s}, i.e., we show how if the reconstruction conjecture holds, it is a most-expressive representation of undirected edge-unattributed graphs. Finally, we show its quadratic number of parameters, exponential computational complexity, and relation to existing work.

\section{Reconstruction Graph Neural Networks}\label{sec:gnn}

Although \recmodel{s}  provide a most-expressive representation for undirected, unattributed-edge graphs, they are impractical due to their computational cost. 
Similarly, \kmodel{s} are not scalable since increasing their expressive power requires computing most-expressive representations of larger $k$-size subgraphs. Hence, to circumvent the computational cost, we replace the most-expressive representations of subgraphs from \kmodel{s} with GNN representations, resulting in what we name \emph{\gnn{s}}. This change allows for scaling the model to larger subgraph sizes, such as $n-1$, $n-2$, \dots, et cetera. 

Since, in the general case, graph reconstruction assumes most-expressive representations of subgraphs, it cannot capture \gnn{s}' expressive power directly. Hence, we provide a theoretical characterization of the expressive power of \gnn{s} by coupling graph reconstruction and the GNN expressive power characterization based on the $1$-WL algorithm. Nevertheless, in~\cref{gnnreco}, we devise conditions under which \gnn{s} have the same power as \kmodel{s}. Finally, we show how graph reconstruction can act as a (provably) powerful inductive bias for invariances to vertex removals, which boosts the performance of GNNs even in tasks where all graphs are already distinguishable by them (see \Cref{supp:exp}). We refer to \Cref{supp:gnn} for a discussion on the model's relation to existing work.

Formally, let $ f_{\bW} \colon \cup^{\infty}_{m=1} \real^{m \times d} \to \real^{t}$ be a (row-wise) permutation invariant function and $h_{\bW}^{\text{GNN}} \colon \cG \to \real^{1 \times d} $ a GNN representation. Then, for $3 \leq k < |V(G)|$,
a \gnn{} takes the form 
\begin{equation*}
    r_{\bW}^{(k, \text{GNN})}(G) \!=\! f_{\bW_1}\!\left( \concat{\lms h_{\bW_2}^{\text{GNN}}(G[S]) \colon S \in \cS^{(k)} \rms } \right)\!, 
\end{equation*}
with parameters $\bW = \{ \bW_1, \bW_2 \} $, where $\cS^{(k)}$ is the set of all $k$-size subsets of $V(G)$, and $\textsc{Concat}$ is row-wise concatenation in some arbitrary order. 

\xhdr{Approximating $r_{\bW}^{(k, \text{GNN})}$.}  By design, \gnn{s} require computing GNN representations for all $k$-vertex subgraphs, which might not be feasible for large graphs or datasets. To address this, we discuss a direction to circumvent computing all subgraphs, i.e., approximating $r_{\bW}^{(k, \text{GNN})}$ by sampling. 

One possible choice for $f_{\bW}$ is Deep Sets~\citep{zaheer2017deep}, which we use for the experiments in~\Cref{sec:exp}, where the representation is a sum decomposition taking the form $r_{\bW}^{(k, \text{GNN})}(G) = \rho_{\bW_{1}}\bigg( \sum_{S \in \cS^{(k)}} \phi_{\bW_{2}} \Big( h_{\bW_3}^{\text{GNN}}(G[S])\Big) \bigg)$,
where $\rho_{\bW_{1}}$ and $\phi_{\bW_{2}}$ are permutation sensitive functions, such as feed-forward networks. 
We can learn the \gnn{} model over a training dataset $ \cD^{(\text{tr})} := \{ (G_i, y_i) \}_{i=1}^{N^{(\text{tr})}}$w and a loss function $l$ by minimizing the empirical risk
\begin{equation}\label{eq:loss} 
\widehat{\cR}_k(\cD^{(\text{tr})};\bW_1, \bW_2, \bW_3) = \frac{1}{N^{\text{tr}}}  \sum_{i=1}^{N^{\text{tr}}} l\big(r_{\bW}^{(k, \text{GNN})}(G_i), y_i\big).
\end{equation}
\Cref{eq:loss} is impractical for all but the smallest graphs, since $r_{\bW}^{(k, \text{GNN})}$ is a sum over all $k$-vertex induced subgraphs $\cS^{(k)}$ of $G$.
Hence, we approximate $r_{\bW}^{(k, \text{GNN})}$ using
a sample $\cS_B^{(k)} \subset \cS^{(k)}$ drawn uniformly at random at every gradient step, i.e., $ \widehat{r}_{\bW}^{(k, \text{GNN})}(G) = \rho_{\bW_{1}}\Big(|\cS^{(k)}|/|\cS_B^{(k)}| \sum_{S \in \cS_B^{(k)}} \phi_{\bW_{2}} \big( h_{\bW_3}^{(\text{GNN})}(G[S]) \big) \Big)$.
Due to non-linearities in $\rho_{\bW_{1}}$ and $l$, plugging $\widehat{r}^{(k, \text{GNN})}_{\bW}$ into \Cref{eq:loss} does not provide us with an unbiased estimate of $\widehat{\cR}_k$. However, if $l(\rho_{\bW_1}(a),y) $ is convex in $a$, in expectation we will be minimizing a proper upper bound of our loss, i.e., $ 1/N^{\text{tr}}  \sum_{i=1}^{N^{\text{tr}}} l\big(r_{\bW}^{(k, \text{GNN})}(G_i), y_i\big) \leq  1/N^{\text{tr}} \sum_{i=1}^{N^{\text{tr}}} l\big(\widehat{r}_{\bW}^{(k, \text{GNN})}(G_i), y_i\big)$.
In practice, many models rely on this approximation and provide scalable and reliable training procedures, cf. \citep{murphy2019janossy,murphy2019relational, zaheer2017deep, hinton2012improving}.

\subsection{Expressive power}\label{sec:exppower}

Now, we analyze the expressive power of \gnn{s}. It is clear that \gnn{s} $\preceq$ \kmodel{s}, however the relationship between \gnn{s} and GNNs is not that straightforward. At first, one expects that there exists a well-defined hierarchy---such as the one in \kmodel{s} (see \Cref{obs:kdeck})---between GNNs, $(n-1)$-Reconstruction GNNs, $(n-2)$-Reconstruction GNNs, and so on. However, \emph{there is no such hierarchy}, as we see next.

\xhdr{Are GNNs more expressive than \gnn{s}?}
It is well-known that GNNs cannot distinguish regular graphs~\citep{Arv+2015,Mor+2019}. By leveraging the fact that regular graphs are reconstructible~\citep{kellytrees}, we show that cycles and circular skip link (CSL) graphs---two classes of regular graphs---can indeed be distinguished by  \gnn{s}, implying that \gnn{s} are not less expressive than GNNs. We start by showing that \gnn{s} can distinguish the class of cycle graphs. 

\begin{theorem}[\gnn{s} can distinguish cycles] \label{thm:cycle} Let $G \in \mathfrak{G}$ be a cycle graph with $n$ vertices and $k:=n-\ell$. An $(n-\ell)$-Reconstruction GNN assigns a unique representation to $G$ if \\i) $ \ell < (1+o(1)) \Big(  \frac{2\log n}{ \log \log n }  \Big)^{1/2}$ and ii) $ n \geq (\ell -\log\ell + 1) \Big( \frac{e + e\log\ell + e + 1}{(\ell-1) \log \ell - 1  }\Big) + 1$ hold.
\end{theorem}
The following results shows that \gnn{s} can distinguish the class of CSL graphs.
\begin{theorem}[\gnn{s} can distinguish CSL graphs] \label{thm:csl} 
Let $G,H \in \mathfrak{G}$ be two non-isomorphic circular skip link (CSL) graphs (a class of 4-regular graphs, cf.  \citep{Che+2020,murphy2019relational}). Then, $(n-1)$-Reconstruction GNNs can distinguish $G$ and $H$.
\end{theorem}
\vspace{-0.5em}
Hence, if the conditions in~\cref{{thm:cycle}} hold, GNNs $\not \preceq$ $(n-\ell)$-Reconstruction GNNs.
\Cref{fig:rec-gnn} (cf. \Cref{supp:gnn}) depicts how \gnn{s} can distinguish a graph that GNNs cannot. The process essentially breaks the local symmetries that make GNNs struggle by removing one (or a few) vertices from the graph. By doing so, we arrive at distinguishable subgraphs. Since we can reconstruct the original graph with its unique subgraph representations, we can identify it. See \Cref{supp:gnn} for the complete proofs of \Cref{thm:cycle,thm:csl}.

\xhdr{Are GNNs less expressive than \gnn{s}?} 
We now show that GNNs can distinguish graphs that \gnn{s} with small $k$ cannot. We start with \Cref{prop:gnn-smallk} stating that there exist some graphs that GNNs can distinguish which \gnn{s} with small $k$ cannot.

\begin{proposition}\label{prop:gnn-smallk} 
     GNNs $ \not \preceq $ \gnn{s} for $k \leq \ceil{n/2} $.
\end{proposition}
\vspace{-0.5em}
On the other hand, the analysis is more interesting for larger subgraph sizes, e.g., $n-1$, where there are no known examples of (undirected, edge-unattributed) non-reconstructible graphs. There are graphs distinguishable by GNNs with at least one subgraph not distinguishable by them; see \Cref{supp:gnn}. However, the analysis is whether the multiset of all subgraphs' representations can distinguish the original graph. Since we could not find any counter-examples, we conjecture that every graph distinguishable by a GNN is also distinguishable by a \gnn{} with $k=n-1$ or possibly more generally with any $k$ close enough to $n$. In \Cref{supp:gnn}, we state and discuss the conjecture, which we name WL reconstruction conjecture. If true, the conjecture implies GNNs $\prec$ $(n-1)$-Reconstruction GNNs. Moreover, if we use the original GNN representation together with \gnn{s}, \Cref{thm:cycle,thm:csl} imply that the resulting model is strictly more powerful than the original GNN. 

\xhdr{Are \gnn{s} less expressive than higher-order ($\kappa$-WL)
GNNs?} 

Recently a line of work, e.g.,~\cite{azizian2020characterizing,Mar+2019a,Mor+2019b}, explored higher-order GNNs aligning with the $\kappa$-WL hierarchy. Such architectures have, in principle, the same power as the $\kappa$-WL algorithm in distinguishing non-isomorphic graphs. Hence, one might wonder how \gnn{s} stack up to $\kappa$-WL-based algorithms. The following result shows that pairs of non-isomorphic graphs exist that a $(n-2)$-Reconstruction GNN can distinguish but the $2$-WL cannot.
\begin{proposition}\label{prop:2wl}
Let $2$-GNNs be neural architectures with the same expressiveness as the $2$-WL algorithm. Then,  $(n-2)\text{-Reconstruction GNN} \not \preceq 2\text{-GNN}s \equiv 2\text{-WL}$.
 \end{proposition}
\vspace{-0.5em}
As a result of \Cref{prop:2wl}, using a $(n-2)$-Reconstruction GNN representation together with a 2-GNN increases the original 2-GNN's expressive power.

\subsection{Reconstruction as a powerful extra invariance for general graphs}\label{sec:inv}
An essential feature of modern machine learning models is capturing invariances of the problem of interest~\citep{lyle2020benefits}. It reduces degrees of freedom while allowing for better generalization~\citep{bloem2020probabilistic,lyle2020benefits}. GRL is predicated on invariance to vertex permutations, i.e., assigning the same representation to isomorphic graphs.
But are there other invariances that could improve generalization error?

\xhdr{$k$-reconstruction is an extra invariance.} Let $P( G, Y )$ be the joint probability of observing a graph $G$ with label $Y$. Any $k$-reconstruction-based model, such as \kmodel{s} and \gnn{s}, by definition assumes $P( G, Y )$ to be invariant to the $k$-deck, i.e., $P( G, Y ) = P( H, Y ) $ if $\cD_k(G)=\cD_k(H)$. 
Hence, our neural architectures for \kmodel{s} and \gnn{s} directly define this extra invariance beyond permutation invariance.
{\em How we do know it is an extra invariance and not a consequence of permutation invariance?}
It does not hold on directed graphs~\cite{stockmeyer1981census}, where permutation invariance still holds.

\xhdr{Hereditary property variance reduction.} We now show that the invariance imposed by $k$-reconstruction helps in tasks based on \textit{hereditary properties}~\citep{borowiecki1997survey}. A graph property $\mu(G)$ is called hereditary if it is invariant to vertex removals, \textit{i.e.} $\mu(G) = \mu(G[V(G)-v]) $ for every $ v \in V(G)$ and $G \in \cG$. 
By induction the property is invariant to every size subgraph, i.e., $\mu(G) = \mu(G[S]) $ for every $ S \in \cS^{(k)}, k \in [n] $ where $\cS^{(k)}$ is the set of all $k$-size subsets of $V(G)$. Here, the property is invariant to any given subgraph. For example, every subgraph of a planar graph is also planar, every subgraph of an acyclic graph is also acyclic, any subgraph of a $j$-colorable graph is also $j$-colorable.
A more practically interesting (weaker) invariance would be invariance to a few vertex removals.
Next we define \textit{$\delta$-hereditary properties} (a special case of a $\preceq$-hereditary property). 
In short, a property is $\delta$-hereditary if it is a hereditary property for graphs with more than $\delta$ vertices.

\begin{definition}[$\delta$-hereditary property] A graph property $\mu \colon \cG \to \cY $ is said to be $\delta$-hereditary if $\mu(G) = \mu(G[V(G)-v]), \: \forall \: v \in V(G), G \in \{  H \in \cG : |V(H)| > \delta \}$. That is, $\mu$ is uniform in $G$ and all subgraphs of $G$ with more than $\delta$ vertices.
\end{definition}
\vspace{-0.5em}

Consider the task of predicting $Y|G := \mu(G)$.
\Cref{thm:var} shows that \gnn{s} is an invariance that reduces the variance of the empirical risk associated with $\delta$-hereditary property tasks. See \Cref{supp:gnn} for the proof.

\begin{theorem}[\gnn{s} for variance reduction of  $\delta$-hereditary tasks] \label{thm:var} Let $P(G,Y)$ be a $\delta$-hereditary distribution, i.e., $Y := \mu(G)$ where $\mu$ is a $\delta$-hereditary property. Further, let $P(G,Y) = 0$ for all $G \in \cG$ with $|V(G)| \geq \delta + \ell $, $\ell > 0$. Then, for \gnn{s} taking the form $\rho_{\bW_{1}}\bigg( 1/|\cS^{(k)}| \sum_{S \in \cS^{(k)}} \phi_{\bW_{2}} \Big( h_{\bW_3}^{\text{GNN}}(G[S])\Big) \bigg)$, if $l(\rho_{\bW_1}(a),y)$ is convex in $a$, we have
$$ \text{Var} [ \widehat{\cR}_k ]  \leq \text{Var} [ \widehat{\cR}_{\text{GNN}} ], $$
where $\widehat{\cR}_k$ is the empirical risk of \gnn{s} with $k:= n - \ell$ (cf.\ \Cref{eq:loss}) and $\widehat{\cR}_{\text{GNN}}$ is the empirical risk of GNNs.
\end{theorem}

\vspace{-1em}

\vspace{-0.em}
\section{Experimental Evaluation}\label{sec:exp}
\vspace{-0.5em}
In this section, we investigate the benefits of \gnn{s} against GNN baselines on both synthetic and real-world tasks. Concretely, we address the following questions:\\
\xhdr{Q1.} Does the increase in expressive power from reconstruction (cf. \Cref{sec:exppower}) make \gnn{s} solve graph property tasks not originally solvable by GNNs? \\
\xhdr{Q2.} Can reconstruction boost the original GNNs performance on real-world tasks? If so, why?\\
\xhdr{Q3.} What is the influence of the subgraph size in both graph property and real-world tasks?

\xhdr{Synthetic graph property datasets.} For \textbf{Q1} and \textbf{Q3}, we chose the synthetic graph property tasks in \Cref{tab:3}, for which GNNs are provably incapable to solve due to their limited expressive power~\citep{Gar+2020, Mur+2019}. The tasks are \textsc{csl}~\citep{Dwi+2020}, where we classify CSL graphs, the cycle detection tasks \textsc{4 cycles}, \textsc{6 cycles} and \textsc{8 cycles}~\citep{Vig+2020} and the multi-task regression from \citet{Cor+2020}, where we want to determine whether a graph is connected, its diameter and its spectral radius. See~\Cref{supp:datasets} for datasets statistics.\\
\xhdr{Real-world datasets.} To address \textbf{Q2} and \textbf{Q3}, we evaluated \gnn{s} on a diverse set of large-scale, standard benchmark instances~\cite{hu2020ogb,Mor+2020}. Specifically, we used the \textsc{zinc} (10K)~\citep{Dwi+2020}, \textsc{alchemy} (10K) ~\citep{Che+2020}, \textsc{ogbg-molfreesolv}, \textsc{ogbg-molesol}, and  \textsc{ogbg-mollipo}~\citep{hu2020ogb} regression datasets. For the case of graph classification, we used \textsc{ogbg-molhiv}, \textsc{ogbg-molpcba}, \textsc{ogbg-tox21}, and \textsc{ogbg-toxcast}~\citep{hu2020ogb}. See~\Cref{supp:datasets} for datasets statistics.\\
\xhdr{Neural architectures.} We used the GIN~\citep{Xu+2018}, GCN~\citep{Kip+2017}, and the PNA~\citep{Cor+2020} architectures as GNN baselines. We always replicated the exact architectures from the original paper, building on the respective PyTorch Geometric implementation~\cite{Fey+2019}.
For the \textsc{ogbg} regression datasets, we noticed how using a jumping knowledge layer yields better validation and test results for GIN and GCN. Thus we made this small change. 
For each of these three architectures, we implemented \gnn{s} for $k$ in $\{ n-1, n-2, n-3, \ceil{n/2} \}$ using a Deep Sets function~\citep{zaheer2017deep} over the \emph{exact same original GNN architecture}. For more details, see \Cref{supp:exp}.\\
\xhdr{Experimental setup.} To establish fair comparisons, we retain all hyperparameters and training procedures from the original GNNs to train the corresponding \gnn{s}. 
\Cref{tab:3,tab:1} and \cref{tab:2_appendix} in~\cref{addres} present results with the same number of runs as previous work~\citep{ Cor+2020,Dwi+2020,hu2020ogb,Morris2020b,Vig+2020}, i.e., five for all datasets execpt the \textsc{ogbg} datasets, where we use ten runs. For more details, such as the number of subgraphs sampled for each \gnn{ } and each dataset, see \Cref{supp:exp}.\\
\xhdr{Non-GNN baselines.} For the graph property tasks, original work used vertex identifiers or laplacian embeddings to make GNNs solve them. This trick is effective for the tasks but violates an important premise of graph representations, invariance to vertex permutations. To illustrate this line of work, we compare against Positional GIN, which uses Laplacian embeddings~\citep{Dwi+2020} for the \textsc{csl} task and vertex identifiers for the others~\citep{Vig+2020,Cor+2020}. To compare against other methods that like \gnn{s} are invariant to vertex permutations and increase the expressive power of GNNs, we compare against Ring-GNNs~\citep{Che+2019} and (3-WL) PPGNs~\citep{Mar+2019}. For real-world tasks, \cref{tab:2_appendix} in~\cref{addres} shows the results from GRL alternatives that incorporate higher-order representations in different ways, LRP~\citep{Che+2019}, GSN~\citep{bouritsas2020improving},  $\delta$-2-LGNN~\citep{Morris2020b}, and SMP~\citep{Vig+2020}.

All results are fully reproducible from the source and are available at \url{https://github.com/PurdueMINDS/reconstruction-gnns}.

\xhdr{Results and discussion.}
\begin{table*}[t]
\caption{Synthetic graph property tasks.  We highlight in green \gnn{s} boosting the original GNN architecture. $^\dagger$: Std. not reported in original work. $^+$: Laplacian embeddings used as positional features. $^*$: vertex identifiers used as positional features. \label{tab:3}\vspace{-1em}
}
\centering
\resizebox{1.0\textwidth}{!}{\renewcommand{\arraystretch}{.9}
\begin{tabular}[t]{@{}m{0.05em}lrrrrrrrc@{}}
\toprule
 &&&& &  & \multicolumn{3}{c}{\textbf{Multi-task}} & \multirow{3}{*}{Invariant to} \\
 \cmidrule{7-9}
             && \multicolumn{1}{c}{\textsc{csl}} & \multicolumn{1}{c}{\textsc{4 cycles}} & \multicolumn{1}{c}{\textsc{6 cycles}} & \multicolumn{1}{c}{\textsc{8 cycles}} & \multicolumn{1}{c}{\textsc{connectivity}} & \multicolumn{1}{c}{\textsc{diameter}} & \multicolumn{1}{c}{\textsc{spectral radius}}
             \\
             && \multicolumn{1}{c}{(Accuracy \% \%) $\uparrow$} & \multicolumn{1}{c}{(Accuracy \% \%) $\uparrow$} & \multicolumn{1}{c}{(Accuracy \%) $\uparrow$ } & \multicolumn{1}{c}{(Accuracy \%) $\uparrow$} &  \multicolumn{1}{c}{($\log$ MSE) $\downarrow$} & \multicolumn{1}{c}{($\log$ MSE) $\downarrow$}  & \multicolumn{1}{c}{($\log$ MSE) $\downarrow$} & vertex permutations?
             \\
\cmidrule{2-10}
& \textbf{GIN} (orig.)    &  4.66 \footnotesize{ $\pm$ 4.00}    &  93.0$^\dagger$      &  92.7$^\dagger$           &   92.5$^\dagger$      & -3.419 \footnotesize{ $\pm$ 0.320}      & 0.588 \footnotesize{ $\pm$ 0.354} & -2.130 \footnotesize{ $\pm$ 1.396 }  & \textcolor{dkgreen}{\CheckmarkBold} \\
\parbox[t]{2mm}{\multirow{4}{*}{\rotatebox[origin=c]{90}{{Reconstr.}}}}
& \footnotesize{$(n-1)$}      &   \colorbox{pastelgreen}{  88.66 \footnotesize{ $\pm$ 22.66 }}   &     \colorbox{pastelgreen}{  95.17 \footnotesize{ $\pm$ 4.91 }}        &        \colorbox{pastelgreen}{  97.35 \footnotesize{ $\pm$ 0.74 }}           &   \colorbox{pastelgreen}{  94.69 \footnotesize{ $\pm$ 2.34 }}           &      \colorbox{pastelgreen}{  -3.575 \footnotesize{ $\pm$ 0.395 }}         &  \colorbox{pastelgreen}{  -0.195 \footnotesize{ $\pm$ 0.714 }}    &  \colorbox{pastelgreen}{  -2.732 \footnotesize{ $\pm$ 0.793 }}    & \textcolor{dkgreen}{\CheckmarkBold}   \\

& \footnotesize{$(n-2)$}      &  \colorbox{pastelgreen}{  78.66 \footnotesize{ $\pm$ 22.17 }}    &       \colorbox{pastelgreen}{  94.06 \footnotesize{ $\pm$ 5.10 }}         &   \colorbox{pastelgreen}{  97.50 \footnotesize{ $\pm$ 0.72 }}         &   \colorbox{pastelgreen}{  95.04 \footnotesize{ $\pm$ 2.69 }}       &      \colorbox{pastelgreen}{  -3.799 \footnotesize{ $\pm$ 0.187 }}        &   \colorbox{pastelgreen}{  -0.207 \footnotesize{ $\pm$ 0.381 }}   &  \colorbox{pastelgreen}{  -2.344 \footnotesize{ $\pm$ 0.569 }}   & \textcolor{dkgreen}{\CheckmarkBold}  \\

& \footnotesize{$(n-3)$}       &  \colorbox{pastelgreen}{  73.33 \footnotesize{ $\pm$ 16.19 }}    &      \colorbox{pastelgreen}{  96.61 \footnotesize{ $\pm$ 1.40 }}        &     \colorbox{pastelgreen}{  97.84 \footnotesize{ $\pm$ 1.37 }}        &   \colorbox{pastelgreen}{  94.48 \footnotesize{ $\pm$ 2.13 }}      &     \colorbox{pastelgreen}{  -3.779 \footnotesize{ $\pm$ 0.064 }}         &   \colorbox{pastelgreen}{  0.105 \footnotesize{ $\pm$ 0.225 }}   & -1.908 \footnotesize{ $\pm$ 0.860 }  & \textcolor{dkgreen}{\CheckmarkBold}  \\

& \footnotesize{$\ceil{n/2}$} &      \colorbox{pastelgreen}{  40.66 \footnotesize{ $\pm$ 9.04 }}            &  75.13 \footnotesize{ $\pm$ 0.26 }   &     63.28 \footnotesize{ $\pm$ 0.59 }              &     63.53 \footnotesize{ $\pm$ 1.14 }       &   \colorbox{pastelgreen}{  -3.765 \footnotesize{ $\pm$ 0.083 }}  &  \colorbox{pastelgreen}{  0.564 \footnotesize{ $\pm$ 0.025 }}  &  -2.130 \footnotesize{ $\pm$ 0.166 } &\textcolor{dkgreen}{\CheckmarkBold} \\

\cmidrule{2-10}

& \textbf{GCN}(orig.)    &  6.66 \footnotesize{ $\pm$ 2.10}        & 98.336 \footnotesize{ $\pm$ 0.24}            &  95.73 \footnotesize{ $\pm$ 2.72} &  87.14 \footnotesize{ $\pm$ 12.73 }      & -3.781 \footnotesize{$\pm$ 0.075}      & 0.087 \footnotesize{ $\pm$ 0.186}  &  -2.204 \footnotesize{ $\pm$ 0.362} & \textcolor{dkgreen}{\CheckmarkBold} \\
\parbox[t]{2mm}{\multirow{4}{*}{\rotatebox[origin=c]{90}{{Reconstr. }}}}
& \footnotesize{$(n-1)$}      & \colorbox{pastelgreen}{  \textbf{100.00} \footnotesize{ $\pm$ 0.00 }}    &    \colorbox{pastelgreen}{  99.00 \footnotesize{ $\pm$ 0.10 }}      &     \colorbox{pastelgreen}{  97.63 \footnotesize{ $\pm$ 0.19 }}             &       \colorbox{pastelgreen}{  94.99 \footnotesize{ $\pm$ 2.31 }}      &      \colorbox{pastelgreen}{  \textbf{-4.039} \footnotesize{ $\pm$ 0.101 }}         & \colorbox{pastelgreen}{  -1.175 \footnotesize{ $\pm$ 0.425 }}    &  \colorbox{pastelgreen}{  -3.625 \footnotesize{ $\pm$ 0.536 }}   & \textcolor{dkgreen}{\CheckmarkBold}   \\

& \footnotesize{$(n-2)$}      &  \colorbox{pastelgreen}{  \textbf{100.00} \footnotesize{ $\pm$ 0.00 }}    &      \colorbox{pastelgreen}{  98.77 \footnotesize{ $\pm$ 0.61 }}        &     \colorbox{pastelgreen}{  97.89 \footnotesize{ $\pm$ 0.69 }}        &   \colorbox{pastelgreen}{  97.82 \footnotesize{ $\pm$ 1.10 }}      &      \colorbox{pastelgreen}{  -3.970 \footnotesize{ $\pm$ 0.059 }}        &   \colorbox{pastelgreen}{  -0.577 \footnotesize{ $\pm$ 0.135 }}    & \colorbox{pastelgreen}{  -3.397 \footnotesize{ $\pm$ 0.273 }} & \textcolor{dkgreen}{\CheckmarkBold} \\

& \footnotesize{$(n-3)$}         &   \colorbox{pastelgreen}{  96.00 \footnotesize{ $\pm$ 6.46 }}   &      \colorbox{pastelgreen}{  99.11 \footnotesize{ $\pm$ 0.19 }}         &    \colorbox{pastelgreen}{  98.31 \footnotesize{ $\pm$ 0.52 }}         &     \colorbox{pastelgreen}{  97.18 \footnotesize{ $\pm$ 0.58 }}       &      \colorbox{pastelgreen}{  -3.995 \footnotesize{ $\pm$ 0.031 }}        &   \colorbox{pastelgreen}{  -0.333 \footnotesize{ $\pm$ 0.117 }}    &  \colorbox{pastelgreen}{  -3.105 \footnotesize{ $\pm$ 0.286 }} & \textcolor{dkgreen}{\CheckmarkBold}  \\

& \footnotesize{$\ceil{n/2}$}  &  \colorbox{pastelgreen}{  49.33 \footnotesize{ $\pm$ 7.42 }}     &     75.19 \footnotesize{ $\pm$ 0.19 }        &   66.04 \footnotesize{ $\pm$ 0.59 }                   &      63.66 \footnotesize{ $\pm$ 0.51 }    &       -3.693 \footnotesize{ $\pm$ 0.063 }      &  0.8518 \footnotesize{ $\pm$ 0.016 }   &   -1.838 \footnotesize{ $\pm$ 0.054 } & \textcolor{dkgreen}{\CheckmarkBold}  \\

\cmidrule{2-10}
& \textbf{PNA} (orig.)   & 10.00 \footnotesize{ $\pm$ 2.98}    &    81.59 \footnotesize{$\pm$ 19.86}        &  95.57 \footnotesize{$\pm$ 0.36 }            &  84.81 \footnotesize{$\pm$ 16.48 }      & -3.794 \footnotesize{$\pm$ 0.155}      & -0.605  \footnotesize{ $\pm$ 0.097}  & -3.610 \footnotesize{ $\pm$ 0.137}  & \textcolor{dkgreen}{\CheckmarkBold} \\
\parbox[t]{2mm}{\multirow{4}{*}{\rotatebox[origin=c]{90}{{Reconstr.}}}}
& \footnotesize{$(n-1)$}      &  \colorbox{pastelgreen}{  \textbf{100.00} \footnotesize{ $\pm$ 0.00 }}  &    \colorbox{pastelgreen}{  97.88 \footnotesize{ $\pm$ 2.19}}       &         \colorbox{pastelgreen}{  99.18 \footnotesize{ $\pm$ 0.20  }}              &      \colorbox{pastelgreen}{ 98.92  \footnotesize{ $\pm$ 0.72 }}       &        \colorbox{pastelgreen}{  -3.904 \footnotesize{ $\pm$ 0.001 }}        & \colorbox{pastelgreen}{  -0.765 \footnotesize{ $\pm$ 0.032 }}   & \colorbox{pastelgreen}{  \textbf{-3.954} \footnotesize{ $\pm$ 0.118 }}   & \textcolor{dkgreen}{\CheckmarkBold}   \\

& \footnotesize{$(n-2)$}      &   \colorbox{pastelgreen}{  95.33 \footnotesize{ $\pm$ 7.77 }}   &     \colorbox{pastelgreen}{ \textbf{99.12}  \footnotesize{ $\pm$ 0.28}}        &       \colorbox{pastelgreen}{ 99.10 \footnotesize{ $\pm$ 0.57 }}          &  \colorbox{pastelgreen}{ \textbf{99.22}  \footnotesize{ $\pm$ 0.27 }}      &     -3.781   \footnotesize{ $\pm$ 0.085 }     &    -0.090   \footnotesize{ $\pm$ 0.135 }   &  -3.478   \footnotesize{ $\pm$ 0.206 }  & \textcolor{dkgreen}{\CheckmarkBold} \\

& \footnotesize{$(n-3)$}    &   \colorbox{pastelgreen}{  95.33 \footnotesize{ $\pm$ 5.81 }}   &     \colorbox{pastelgreen}{  89.36  \footnotesize{$\pm$ 0.22 }}        &     \colorbox{pastelgreen}{99.34 \footnotesize{ $\pm$ 0.26 }}        &    \colorbox{pastelgreen}{ 93.92 \footnotesize{ $\pm$ 8.15 }}      &       -3.710  \footnotesize{ $\pm$ 0.209 }    &  0.042 \footnotesize{ $\pm$ 0.047 }   & -3.311 \footnotesize{ $\pm$ 0.067 }  & \textcolor{dkgreen}{\CheckmarkBold}  \\

& \footnotesize{$\ceil{n/2}$}  &   \colorbox{pastelgreen}{  42.66 \footnotesize{ $\pm$ 11.03 }}   &    75.34 \footnotesize{ $\pm$0.18 }       &    65.58 \footnotesize{ $\pm$ 0.95 }        &  64.01 \footnotesize{ $\pm$ 0.30}   &    -2.977 \footnotesize{ $\pm$ 0.065 }        &   1.445 \footnotesize{ $\pm$ 0.037 }  &  -1.073 \footnotesize{ $\pm$ 0.075 } & \textcolor{dkgreen}{\CheckmarkBold}  \\

\cmidrule{2-10}
& Positional GIN &       99.33$^+$ \footnotesize{ $\pm$ 1.33}          &  88.3$^\dagger$  &         96.1$^\dagger$        &    95.3$^\dagger$      &  -1.61$^\dagger$  & \textbf{-2.17}$^\dagger$    & -2.66$^\dagger$  & \textcolor{red}{\XSolidBrush} \\
& Ring-GNN &     10.00  \footnotesize{ $\pm$ 0.00}  & 99.9$^\dagger$  &     \textbf{100.0}$^\dagger$      &    71.4$^\dagger$    &    ---        & ---  & ---  & \textcolor{dkgreen}{\CheckmarkBold}  \\
& PPGN  (3-WL) &    97.80  \footnotesize{ $\pm$ 10.91} &           99.8$^\dagger$  &         87.1$^\dagger$        &    76.5$^\dagger$      &  --- &   ---  &  ---  & \textcolor{dkgreen}{\CheckmarkBold}  \\
\bottomrule
\end{tabular}
} \end{table*}

\xhdr{A1 (Graph property tasks).} \Cref{tab:3} confirms~\Cref{thm:csl}, where the increase in expressive power from reconstruction allows \gnn{s} to distinguish CSL graphs, a task that GNNs cannot solve.
Here, \gnn{s} boost the accuracy of standard GNNs between 10$\times$ and 20$\times$. \Cref{thm:csl} only guarantees GNN expressiveness boosting for $(n-1)$-Reconstruction, but our empirical results also show benefits for $k$-Reconstruction with $k\leq n\! -\! 2$. \Cref{tab:3} also confirms~\Cref{thm:cycle}, where \gnn{s} provide significant accuracy boosts on all  cycle detection tasks (\textsc{4 cycles}, \textsc{6 cycles} and \textsc{8 cycles}).
See~\Cref{appx:result_properties}, for a detailed discussion on results for \textsc{connectivity}, \textsc{diameter}, and \textsc{spectral radius}, which also show boostings.
\begin{table}
\centering
\resizebox{0.9\textwidth}{!}{\renewcommand{\arraystretch}{.9}
\begin{tabular}[t]{@{}m{0.05em}lrrrrrr@{}}
\toprule
             && \multicolumn{1}{c}{\textsc{ogbg-moltox21}} & \multicolumn{1}{c}{\textsc{ogbg-moltoxcast}} & \multicolumn{1}{c}{\textsc{ogbg-molfreesolv}} & \multicolumn{1}{c}{\textsc{ogbg-molesol}} & \multicolumn{1}{c}{\textsc{ogbg-mollipo}} & \multicolumn{1}{c}{\textsc{ogbg-molpcba}}
             \\
             && \multicolumn{1}{c}{(ROC-AUC \%) $\uparrow$} & \multicolumn{1}{c}{(ROC-AUC \%) $\uparrow$} & \multicolumn{1}{c}{(RSMSE) $\downarrow$ } & \multicolumn{1}{c}{(RSMSE) $\downarrow$} &  \multicolumn{1}{c}{(RSMSE) $\downarrow$} & \multicolumn{1}{c}{(AP \%) $\uparrow$}
             \\
\cmidrule{2-8}
& \textbf{GIN} (orig.)   &  74.91   \footnotesize{ $\pm$ 0.51 }    &        63.41   \footnotesize{ $\pm$ 0.74 }       &     2.411 \footnotesize{$\pm$ 0.123 }                         &      1.111   \footnotesize{$\pm$ 0.038 }                   &           0.754   \footnotesize{$\pm$ 0.010 }      &    21.16 \footnotesize{ $\pm$ 0.28 }    \\

\parbox[t]{2mm}{\multirow{4}{*}{\rotatebox[origin=c]{90}{{Reconstr.}}}}
& \footnotesize{$(n-1)$}      &  \colorbox{pastelgreen}{  75.15 \footnotesize{ $\pm$ 1.40 }}    &                          \colorbox{pastelgreen}{  63.95 \footnotesize{ $\pm$ 0.53 }}       &      \colorbox{pastelgreen}{   2.283 \footnotesize{$\pm$ 0.279 }}                      &    \colorbox{pastelgreen}{   1.026  \footnotesize{$\pm$ 0.033 }}                      &      \colorbox{pastelgreen}{ \textbf{0.716}  \footnotesize{$\pm$ 0.020 }}              &  \colorbox{pastelgreen}{   23.60  \footnotesize{$\pm$ 0.02 }}      \\

& \footnotesize{$(n-2)$}      &  \colorbox{pastelgreen}{  \textbf{76.84} \footnotesize{ $\pm$ 0.62 }}        &                   \colorbox{pastelgreen}{  \textbf{65.36} \footnotesize{ $\pm$ 0.49 }}       &      \colorbox{pastelgreen}{   \textbf{2.117} \footnotesize{$\pm$ 0.181 }}                        &      \colorbox{pastelgreen}{  \textbf{1.006}  \footnotesize{$\pm$ 0.030 }}                     &          \colorbox{pastelgreen}{  0.736  \footnotesize{$\pm$ 0.025 }}        &   \colorbox{pastelgreen}{   23.25  \footnotesize{$\pm$ 0.00 }}      \\

& \footnotesize{$(n-3)$}        &     \colorbox{pastelgreen}{  76.78 \footnotesize{ $\pm$ 0.64 }}                      &   \colorbox{pastelgreen}{  64.84 \footnotesize{ $\pm$ 0.71 }}         &         \colorbox{pastelgreen}{   2.370 \footnotesize{$\pm$ 0.326} }                     &    \colorbox{pastelgreen}{   1.055  \footnotesize{$\pm$ 0.031 }}                          &                       \colorbox{pastelgreen}{  0.738  \footnotesize{$\pm$ 0.018 }}    &  \colorbox{pastelgreen}{   23.33  \footnotesize{$\pm$ 0.09 }}   \\

& \footnotesize{$\ceil{n/2}$} &          74.40 \footnotesize{ $\pm$ 0.75 }                 &  62.29  \footnotesize{ $\pm$ 0.28 }   &              2.531    \footnotesize{$\pm$ 0.206 }                 &      1.343   \footnotesize{$\pm$ 0.053 }           &  0.842 \footnotesize{$\pm$ 0.020 }  &  13.50  \footnotesize{$\pm$ 0.32 }   \\

\cmidrule{2-8}
& \textbf{GCN} (orig.)   &  75.29 \footnotesize{ $\pm$ 0.69 }      &    63.54 \footnotesize{ $\pm$ 0.42 }        &                2.417  \footnotesize{$\pm$ 0.178 }              &      1.106     \footnotesize{$\pm$ 0.036 }                  &           0.793  \footnotesize{$\pm$ 0.040~~~ } &  20.20 \footnotesize{ $\pm$ 0.24 }   \\

\parbox[t]{2mm}{\multirow{4}{*}{\rotatebox[origin=c]{90}{{Reconstr.}}}}
& \footnotesize{$(n-1)$}      &   \colorbox{pastelgreen}{  76.46 \footnotesize{ $\pm$ 0.77 }}        &                \colorbox{pastelgreen}{  64.51 \footnotesize{ $\pm$ 0.60 }}              &       2.524 \footnotesize{$\pm$ 0.300 }                     &     \colorbox{pastelgreen}{   1.096  \footnotesize{$\pm$ 0.045 }}                        &       \colorbox{pastelgreen}{  0.760  \footnotesize{$\pm$ 0.015 }}    & \colorbox{pastelgreen}{  21.25 \footnotesize{ $\pm$ 0.25 }}   \\

& \footnotesize{$(n-2)$}      &  \colorbox{pastelgreen}{  75.58 \footnotesize{ $\pm$ 0.99 }}          &                   \colorbox{pastelgreen}{  64.38 \footnotesize{ $\pm$ 0.39 }}            &       2.467 \footnotesize{$\pm$ 0.231 }                            &      \colorbox{pastelgreen}{   1.086  \footnotesize{$\pm$ 0.048 }}                       &   \colorbox{pastelgreen}{  0.766  \footnotesize{$\pm$ 0.025 }}       &            20.10 \footnotesize{$\pm$ 0.08 }         \\

& \footnotesize{$(n-3)$}        & \colorbox{pastelgreen}{  75.88 \footnotesize{ $\pm$ 0.73 }}                         & \colorbox{pastelgreen}{  64.70 \footnotesize{ $\pm$ 0.81 }}        &        \colorbox{pastelgreen}{   2.345 \footnotesize{$\pm$ 0.261 }}                   &     1.114  \footnotesize{$\pm$ 0.047}                   &          \colorbox{pastelgreen}{  0.754  \footnotesize{$\pm$ 0.021 }}           &    19.04 \footnotesize{ $\pm$ 0.03}      \\

& \footnotesize{$\ceil{n/2}$} & 74.03 \footnotesize{ $\pm$ 0.63 }                 & 62.80  \footnotesize{ $\pm$ 0.77 }  &2.599     \footnotesize{$\pm$ 0.161}   &     1.372      \footnotesize{$\pm$ 0.048}                           & 0.835 \footnotesize{$\pm$ 0.020 }  &  11.69  \footnotesize{$\pm$ 1.41 }  \\

\cmidrule{2-8}
& \textbf{PNA} (orig.)    &    74.28  \footnotesize{ $\pm$ 0.52 }    &   62.69 \footnotesize{ $\pm$ 0.63 }          &               2.192  \footnotesize{$\pm$ 0.125 }                &     1.140     \footnotesize{$\pm$ 0.032}                   &                0.759   \footnotesize{$\pm$ 0.017 }     &  25.45  \footnotesize{$\pm$ 0.04 }   \\

\parbox[t]{2mm}{\multirow{4}{*}{\rotatebox[origin=c]{90}{{Reconstr.}}}}
& \footnotesize{$(n-1)$}      &   73.64   \footnotesize{ $\pm$ 0.74 }     &      \colorbox{pastelgreen}{ 64.14   \footnotesize{ $\pm$ 0.76 }}       &        2.341 \footnotesize{$\pm$ 0.070}                     &         1.723 \footnotesize{$\pm$ 0.145}                   &      \colorbox{pastelgreen}{  0.743  \footnotesize{$\pm$ 0.015}   }              &    23.11 \footnotesize{$\pm$ 0.05 }  \\

& \footnotesize{$(n-2)$}      &  \colorbox{pastelgreen}{ 74.89   \footnotesize{ $\pm$ 0.29 }}    &     \colorbox{pastelgreen}{ 65.22   \footnotesize{ $\pm$ 0.47 }}               &        2.298 \footnotesize{$\pm$ 0.115 }                      &       1.392 \footnotesize{$\pm$ 0.272 }                    &        0.794     \footnotesize{$\pm$ 0.065 }         &  22.10 \footnotesize{$\pm$ 0.03 } \\

& \footnotesize{$(n-3)$}        &      \colorbox{pastelgreen}{ 75.10   \footnotesize{ $\pm$ 0.73 }}        &   \colorbox{pastelgreen}{ 65.03   \footnotesize{$\pm$ 0.58 }}     &          \colorbox{pastelgreen}{ 2.133 \footnotesize{$\pm$ 0.086 }}                  &       1.360 \footnotesize{$\pm$ 0.163}  &      0.785     \footnotesize{$\pm$ 0.041}        &        20.05 \footnotesize{$\pm$ 0.15 }   \\

& \footnotesize{$\ceil{n/2}$} & 73.71   \footnotesize{ $\pm$ 0.61 }    &  61.25  \footnotesize{$\pm$ 0.49 }   &        2.185     \footnotesize{$\pm$ 0.231 }              &      1.157  \footnotesize{$\pm$ 0.056}                & 0.843 \footnotesize{$\pm$ 0.018 } &  12.33 \footnotesize{$\pm$ 1.20 }   \\

\bottomrule
\end{tabular}

  \caption{\textsc{ogbg} molecule graph classification and regression tasks. We highlight in green \gnn{s} boosting the original GNN architecture.\label{tab:1}\vspace{-1.0em}}}\end{table}
\\ \xhdr{A2 (Real-world tasks).} \Cref{tab:1} and \cref{tab:2_appendix} in~\cref{addres} show that applying $k$-reconstruction to GNNs significantly boosts their performance across all eight real-world tasks. 
In particular, in \Cref{tab:1} we see a boost of up to 5\% while achieving the best results in five out of six datasets. The $(n-2)$-reconstruction applied to GIN gives the best results in the \textsc{ogbg} tasks, with the exception of \textsc{ogbg-mollipo} and \textsc{ogbg-molpcba} where $(n-1)$-reconstruction performs better. The only settings where we did not get any boost were PNA for \textsc{ogbg-molesol} and \textsc{ogbg-molpcba}. \Cref{tab:2_appendix} in~\cref{addres} also shows consistent boost in GNNs' performance of up to 25\% in other datasets.
On \textsc{zinc}, $k$-Reconstruction yields better results than the higher-order alternatives LRP and $\delta$-2-LGNN. While GSN gives the best \textsc{zinc} results, we note that GSN requires application-specific features. In \textsc{ogbg-molhiv}, $k$-reconstruction is able to boost both GIN and GCN. The results in \Cref{supp:exp} show that nearly $ 100\%$ of the graphs in our real-world datasets are distinguishable by the $1$-WL algorithm, thus we can conclude that traditional GNNs are expressive enough for all our real-world tasks. Hence, real-world boosts of reconstruction over GNNs can be attributed to the gains from invariances to vertex removals (cf. \Cref{sec:inv}) rather than the boost in expressive power (cf. \Cref{sec:exppower}).

\textbf{A3 (Subgraph sizes).} 
Overall we observe that removing one vertex ($k\!=\!n\!-\!1$) is enough to improve the performance of GNNs in most experiments.
At the other extreme end of vertex removals, $k\!=\!\ceil{n/2}$, there is a significant loss in expressiveness compared to the original GNN.
In most real-world tasks \Cref{tab:1} and \cref{tab:2_appendix} in~\cref{addres} show a variety of performance boosts also with $k \in \{n\!-\!2, n\!-\!3\}$. For GCN and PNA in \textsc{ogbg-molesol}, specifically, we only see $k$-Reconstruction boosts over smaller subgraphs such as $n-3$, which might be due to the task's need of more invariance to vertex removals (cf. \Cref{sec:inv}). 
In the graph property tasks (\Cref{tab:3}), we see significant boosts also for $k \in \{n\!-\!2, n\!-\!3\}$ in all models across most tasks, except PNA. However, as in real-world tasks the extreme case of small subgraphs $k=\ceil{n/2}$ significantly harms the ability to solve tasks with \gnn{s}. 
\vspace{-1em}

\section{Conclusions}
\vspace{-0.5em}
Our work connected graph ($k$-)reconstruction and modern GRL. We first showed how such connection results in two natural expressive graph representation classes. To make our models practical, we combined insights from graph reconstruction and GNNs, resulting in \gnn{s}. Our theory shows that reconstruction boosts the expressiveness of GNNs and has a lower-variance risk estimator in distributions invariant to vertex removals. Empirically, we showed how the theoretical gains of \gnn{s} translate into practice, solving graph property tasks not originally solvable by GNNs and boosting their performance on real-world tasks.

\section*{Acknowledgements}
This work was funded in part by the National Science Foundation (NSF) awards CAREER IIS-1943364 and CCF-1918483.  Any opinions, findings and conclusions or recommendations expressed in this material are those of the authors and do not necessarily reflect the views of the sponsors.  Christopher Morris is funded by the German Academic Exchange Service (DAAD) through a DAAD IFI postdoctoral scholarship (57515245).
We want to thank our reviewers, who gave excellent suggestions to improve the paper.

\small{
\bibliography{refs}
\bibliographystyle{apalike}
}

\appendix

\newpage

\section{Related work (expanded)}\label{ext_rw}
\xhdr{GNNs.} Recently, graph neural networks~\citep{Gil+2017,Sca+2009} emerged as the most prominent (supervised) GRL architectures. Notable instances of this architecture include, e.g.,~\citep{Duv+2015,Ham+2017,Vel+2018}, and the spectral approaches proposed in, e.g.,~\citep{Bru+2014,Def+2015,Kip+2017,Mon+2017}---all of which descend from early work in~\citep{Kir+1995,Mer+2005,Spe+1997,Sca+2009}. Recent extensions and improvements to the GNN framework include approaches to incorporate different local structures (around subgraphs), e.g.,~\citep{Hai+2019,Fla+2020,Jin+2020,Nie+2016,Xu+2018}, novel techniques for pooling vertex representations in order perform graph classification, e.g.,~\citep{Can+2018,Gao+2019,Yin+2018,Zha+2018}, incorporating distance information~\citep{You+2019}, and non-euclidian geometry approaches~\citep{Cha+2019}. Moreover, recently empirical studies on neighborhood aggregation functions for continuous vertex features~\citep{Cor+2020}, edge-based GNNs leveraging physical knowledge~\citep{And+2019,Kli+2020}, and sparsification methods~\citep{Ron+2020} emerged. A survey of recent advancements in GNN techniques can be found, e.g., in~\citep{Cha+2020,Wu+2019,Zho+2018}. 

\xhdr{Limits of GNNs.} \citet{Chen+2020} study the substructure counting abilities of GNNs.  \citet{Das+2020,Abb+2020} investigate the connection between random coloring and universality. Recent works have extended GNNs' expressive power by encoding vertex identifiers~\citep{murphy2019relational, Vig+2020}, adding random features~\citep{Sat+2020}, using higher-order topology as features~\citep{bouritsas2020improving}, considering simplicial complexes~\citep{,Alb+2019,bodnar2021weisfeiler}, encoding ego-networks~\citep{you2021identity}, and encoding distance information~\citep{li2020distance}. Although these works increase the expressiveness of GNNs, their generalization abilities are understood to a lesser extent. Further, works such as \citet[Lemma 6]{Vig+2020} and the most recent \citet{beaini2020directional} and \citet{bodnar2021weisfeiler} prove the boost in expressiveness with a single pair of graphs, giving no insights into the extent of their expressive power or their generalization abilities. For clarity, throughout this work, we use the term GNNs to denote the class of message-passing architectures limited by the $1$-WL algorithm, where the class of distinguishable graphs is well understood~\citep{Arv+2015}.

\section{Notation (expanded)}\label{notation}

As usual, let $[n] = \{ 1, \dotsc, n \} \subset \N$ for $n \geq 1$, and let $\{\!\!\{ \dots\}\!\!\}$ denote a multiset. In an abuse of notation, for a set $X$ with $x$ in $X$, we denote by $X - x$ the set $X \setminus \{x\}$.

\xhdr{Graphs.} A \emph{graph} $G$ is a pair $(V,E)$ with a \emph{finite} set of
\emph{vertices} $V$ and a set of \emph{edges} $E \subseteq \{ \{u,v\}
\subseteq V \mid u \neq v \}$. We denote the set of vertices and the set
of edges of $G$ by $V(G)$ and $E(G)$, respectively. For ease of
notation, we denote the edge $\{u,v\}$ in $E(G)$ by $(u,v)$ or
$(v,u)$. In the case of \emph{directed graphs} $E \subseteq \{ (u,v)
\in V \times V \mid u \neq v \}$. An \emph{attributed graph} $G$ is a triple
$(V,E,\alpha)$ with an attribute function $\alpha \colon V(G) \cup E(G) \to \mathbb{R}^a$ for $a > 0$. Then $\alpha(v)$ is an \emph{attribute} of $v$ for $v$ in $V(G) \cup E(G)$. 
The \emph{neighborhood} 
of $v$ in $V(G)$ is denoted by $N(v) = \{ u \in V(G) \mid (v, u) \in E(G) \}$. Unless indicated otherwise, we use $n:=|V(G)|$.

We say that two graphs $G$ and $H$ are isomorphic, $G \simeq H$, if there exists an adjacency preserving bijection $\varphi \colon V(G) \to V(H)$, i.e., $(u,v)$ is in $E(G)$ if and only if $(\varphi(u),\varphi(v))$ is in $E(H)$, and call $\varphi$ an \emph{isomorphism} from $G$ to $H$. If the graphs have vertex or edge attributes, the isomorphism is additionally required to match these attributes accordingly.

We denote the set of all finite and simple graphs by $\cG$. The subset of $\cG$ without edge attributes is denoted $\mathfrak{G} \subset \cG$. Further, we denote the isomorphism type, i.e., the equivalence class of the isomorphism relation, of a graph $G$ as $\cI(G)$. Let $S \subseteq V(G)$, then $G[S]$ is the induced subgraph with edge set $E(G)[S] = \{ S^2 \cap E(G) \}$. We will refer to induced subgraphs simply as subgraphs in this work.

\section{More on reconstruction}\label{supp:rec}

After formulating the Reconstruction Conjecture, it is natural to wonder whether it stands for other relational structures, such as directed graphs. Interestingly, directed graphs, hypergraphs, and infinite graphs are not reconstructible~\citep{bondymanual,stockmeyer1977falsity}. Thus, in particular the Reconstruction Conjecture does not hold for the class $\cG$.

Another question is how many cards from the deck are sufficient to reconstruct a graph. \citet{bollobas1990almost} show that almost every graph, in a probabilistic sense, can be reconstructed with only three subgraphs from the deck. For example, the graph shown in \Cref{fig:deck} (\Cref{sec:prelims}) is reconstructible from the three leftmost cards.

For an extensive survey on reconstruction, we refer the reader to ~\citet{bondymanual,godsil1993algebraic}. From there, we highlight a significant result, Kelly's Lemma (cf. \Cref{lemma:kelly}). In short, the lemma states that the deck of a graph completely defines its subgraph count of every size.

\begin{lemma}[Kelly's Lemma~\citep{kellytrees}] \label{lemma:kelly} Let $\nu(H,G)$ be the number of copies of $H$ in $G$. For any pair of graphs $G,H \in \cG$ with $V(G) > V(H)$, $\nu(H,G)$ is reconstructible.
\end{lemma}
In fact, its proof is very simple once we realize every subgraph $H$ appears in exactly $(|V(G)|-|V(H|)$ cards from the deck, i.e., $$\nu(G,H) = \sum_{v \in V(G)} \frac{\nu(G[V(G)-v],H)}{(|V(G)|-|V(H|)}.$$

\citet{manvel1974some} started the study of graph reconstruction with the $k$-deck, which has been recently reviewed by  \citet{kostochka2019reconstruction} and \citet{nydl2001graph}. Related work also refers to $k$-reconstruction as $\ell$-reconstruction~\citep{kostochka2019reconstruction}, where $\ell = n-k$ is the number of deleted vertices from the original graph.

Here, in \Cref{lemma:k-kelly}, we highlight a generalization of Kelly's Lemma (\Cref{lemma:kelly}) established in \citet{nydl2001graph}, where the count of any subgraph of size at most $k$ is $k$-reconstructible.
\begin{lemma}[\citet{nydl2001graph}] \label{lemma:k-kelly} For any pair of graphs $G,H \in \cG$ with $V(G) > k \geq V(H)$, $\nu(H,G)$ is $k$-reconstructible.
\end{lemma}

\section{More on \kmodel{s}}\label{supp:kmodel}

Here, we give more background on \kmodel{s}. 

\subsection{Properties}

We start by showing how the $k$-ary Relational Pooling framework~\cite{murphy2019relational} is a specific case of \kmodel{s} and thus limited by $k$-reconstruction. Then, we show how \kmodel{s} are limited by $k$-GNNs at initialization, which implies that $k$-GNNs~\cite{Mor+2019} at initialization can approximate any $k$-reconstructible function.

\begin{observation}[$k$-ary Relational Pooling $\preceq$ \kmodel{s}]\label{obs:rp}
	
	The $k$-ary pooling approach in the Relational Pooling (RP) framework~\citep{murphy2019relational} defines a graph representation of the form
	
	$$ h^{(\text{RP})}_\bW(G) = \frac{1}{\binom{n}{k}}\sum_{S \in \cS^{(k)}} \overrightarrow{h}^{(k)}_\bW(G[S]),$$

	where $\cS^{(k)}$ is the set of all $\binom{n}{k}$ $k$-size subsets of $V$ and $\overrightarrow{h}^{(k)}_\bW(\cdot)$ is a most-expressive graph representation given by the average of a permutation-sensitive universal approximator, e.g., a feed-forward neural network, applied over the $k!$ permutations of the subgraph, accordingly. Thus, $k$-ary RP can be casted as a $k$-Reconstruction Neural Network with $f_\bW$ as mean pooling and $h^{(k)}$ as $\overrightarrow{h}^{(k)}_\bW$. Note that for $k$-ary RP to be as expressive as \kmodel{s}, i.e., $k$-ary RP $\equiv$ \kmodel{s}, we would need to replace the average pooling by a universal multiset approximator or simply add a feed-forward neural network after it.
	
\end{observation}

\begin{observation}[\kmodel{s} $\preceq$ $k$-WL at initialization]\label{obs:kwl}
The $k$-WL test, which limits architectures such as \citet{Morris2020b,Mor+2019,Mar+2019}, at initialization, with zero iteration, considers one-hot encodings of $k$-tuples of vertices. Note that each $k$-size subgraph is completely defined by its corresponding $k!$ vertex tuples. Thus, it follows that $k$-WL with zero iterations is at least as expressive as \kmodel{s}.
Further, by combining \Cref{prop:kmodel} and the result from \Cref{lemma:k-kelly}~\citep{nydl2001graph}, it follows that $k$-WL~\citep{Morris2020b} at initialization can count subgraphs of size $\leq k$, which is a simple proof for the recent result \citep[Theorem 3.7]{Che+2020}.
\end{observation}

Now, we discuss the computational complexity of \kmodel{s} and how to circumvent it through subgraph sampling.

\xhdr{Computational complexity.} As outlined in \Cref{sec:prelims}, we would need subgraphs of size almost $n$ to have a most-expressive representation of graphs with \kmodel{s}. This would imply performing isomorphism testing for arbitrarily large graphs, as in \citet{bouritsas2020improving}, making the model computationally infeasible. 

A graph with $n$ vertices has $\binom{n}{k}$ induced subgraphs of size $k$. Let $\cT_{h^{(k)}}$ be an upper-bound on computing $h^{(k)}$. Thus, computing $r^{(k)}_{\bW}(G)$ would take $\cO(\binom{n}{k}\cT_{h^{(k)}})$ time. Although \citet{Bab+2016} has shown how to do isomorphism testing in quasi-polynomial time, an efficient (polynomial) time algorithm remains unknown. More generally, expressive representations of graphs~\citep{Ker+2019,murphy2019relational} and isomorphism class hashing algorithms~\citep{junttila2007engineering} still require exponential time regarding the graph size. Thus, if we choose a small value for $k$, i.e., $n \gg k$, the $\binom{n}{k}$ factor dominates, while if we choose $k \approx n$ the $\cT_{h^{(k)}}$ factor dominates. In both cases, the time complexity is exponential in $k$, i.e., $\cO(n^k)$.

\subsection{Relation to previous work}
Recently, \citet{bouritsas2020improving} propose using subgraph isomorphism type counts as features of vertices and edges used in a GNN architecture. The authors comment that if the reconstruction conjecture holds, their architecture is most expressive for $k=n-1$. Here, we point out two things. First, their architecture is at least as powerful as $k$-reconstruction. Secondly, the reconstruction conjecture does not hold for directed graphs. Since edge directions can be seen as edge attributes, their architecture is not the most expressive for graphs with attributed edges. Finally, to make their architecture scalable, in practice, the authors choose only specific hand-engineered subgraph types, which makes the model incomparable to $k$-reconstruction.

\subsection{Proof of \Cref{prop:kmodel} }

We start by giving a more formal statement of \Cref{prop:kmodel}. 

Let $f$ be a continuous function over a compact set of $\cG$ and $|| \cdot ||$ the uniform (sup) norm.  \Cref{prop:kmodel} states that for every $\epsilon > 0$ there exists some $\bW_{\epsilon}$ such that $|| f(G) - r^{(k)}_{\bW_{\epsilon}}(G) || < \epsilon $  if and only if $f$ is $k$-reconstructible.

\begin{proof}

Since $h^{(k)}$ is required to be most expressive, we can see the input of \kmodel{s} as a multiset of unique identifiers of isomorphism types. Thus, it follows from \Cref{def:func-krecon}, that $k$-reconstrucible functions can be approximated by $r^{(k)}_{\bW}$. The other direction, i.e., a function can be approximated by $h^{(k)}$ if it is $k$-reconstructible, follows from the 
Stone–Weierstrass theorem, see \citet{zaheer2017deep}.\qedhere
\end{proof}

\section{More on \recmodel{s}}\label{supp:recmodel}

The following result captures the expressive power of \recmodel{s}.

\begin{proposition}\label{prop:fullrecon}
If the functions $f^{k}_{\bW} \text{ for all } k=3,...,n^*$ are universal approximators of multisets~\citep{murphy2019janossy,wagstaff19,zaheer2017deep} and the Reconstruction Conjecture holds, \recmodel{s} can approximate a function if the function is reconstructible.
\end{proposition}

\begin{proof}
We use induction on $|S|$ to show that every subgraph representation in \recmodel{s} is a most expressive representation if the Reconstruction Conjecture holds.
\begin{itemize}
    \item[i)] Base case: $|S|=2$. It follows from the model definition that $r(G[S])$ is a most expressive representation if $|S|=2$.
    \item[ii)] Inductive step: $2<|S|<n^*$. If all subgraph representations in $\lms r{(G[S - v ])} \!\mid\! v \in S \rms$ are most expressive, it follows from \Cref{prop:kmodel} that if $f^{(|S|)}_{\bW}$is a universal approximator of multisets $r(G[S])$ can approximate any reconstructible function. Thus, if the Reconstruction Conjecture holds, $r(G[S])$ can assign a most expressive representation to $G[S]$.
\end{itemize}
It follows then that $r(G[V(G)])$ will be a multiset function $f^{(n^*)}_{\bW}$ of $\lms r{(G[V(G) - v ])} \!\mid\! v \in V(G) \rms$. From \Cref{prop:kmodel}, if $f^{(n^*)}_{\bW}$ is a universal approximator of multisets $r(G[V(G)])$ can approximate any reconstructible function.\qedhere
\end{proof}
It follows from \Cref{prop:fullrecon} that if the Reconstruction Conjecture holds, \recmodel{s} are a most-expressive representation of $\mathfrak{G}_{\leq n^*}^\dagger$.

\xhdr{Number of parameters.} \citet{wagstaff19} shows how a multiset model needs at least $N$ neurons to learn over multisets of size at most $N$. Since our graphs have at most $n^*$ vertices, we can bound the multiset input size of each $\mathbf{f}_{\bW_{k}} \text{ for all } k=3,...,n^*$. Thus, our total number of parameters is $\mathcal{O}(n^{*^2})$.

\xhdr{Computational complexity.} For a graph with $n$ vertices, we need to compute representations of all subgraphs of sizes $2,3,\dots,n$, i.e., $\binom{n}{2} + \binom{n}{3} + \cdots + \binom{n}{n-1}$.  Thus, computing a \recmodel{} representation takes $\mathcal{O}(2^{n^*})$ time.

\xhdr{Relation to previous work.} Unlike \citet{shawe1993symmetries}, the first work proposing reconstruction to build symmetric neural networks for unattributed graphs with a fixed size, we can handle graphs with vertex attributes and of varying sizes of size up to $n^*$. Future work can explore approximate computing methods a \recmodel{} representation, as recently done for the Relational Pooling (RP) framework. Further, in contrast to the most expressive representation in the RP framework, which uses a permutation-sensitive function, \recmodel{s} incorporate graph invariances in the model.

\section{More on \gnn{s}}\label{supp:gnn}

\begin{figure}
	\includegraphics[width=0.3\textwidth]{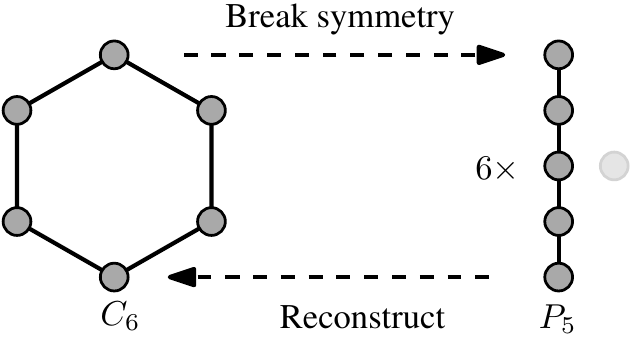}
	\caption{A cycle, undistinguishable by GNNs, and how reconstruction makes it distinguishable.}
	\label{fig:rec-gnn}
\end{figure}

In the following, we give more details on \gnn{s}.

\subsection{Relation to previous work}

Recently, \citet{Gar+2020} showed how many graph properties are not recognizable by GNNs using specific cycle graph examples. Further, most existing work extending GNN architectures to make them more expressive, such as \citet{Vig+2020, li2020distance}, and \citet{murphy2019relational} focus on distinguishing regular graph examples. Finally, \citet{beaini2020directional} used a single pair of planar, non-regular graphs, also used in \citet{Gar+2020}, to prove their method is more expressive than GNNs. Here, we show how the simple idea of graph reconstruction---without changing the original GNN architecture---can extend the GNN's expressivity, distinguishing some classes of regular graphs. 

\subsection{Relating \gnn{s} and \kmodel{s}}\label{gnnreco}

Here, we devise conditions under which \gnn{s} and \kmodel{s} have the same power, using the following definition.
\begin{definition}
Let $\mathcal{D}$ be a distribution on graphs. A graph representation is \emph{$\varepsilon$-universal} with $\varepsilon$ in $[0,1)$ for $\mathcal{D}$ if it assigns, with probability $(1-\varepsilon)$, a unique representation, up to isomorphism, to a graph sampled from $\mathcal{D}$. If a graph representation is $\varepsilon$-universal for all induced $k$-vertex subgraphs of graphs sampled from $\mathcal{D}$, then the representation is \emph{$(\varepsilon,k)$-universal} for $\mathcal{D}$.
\end{definition}

Based on the above definition, we get the following result, relating \gnn{s} and \kmodel{s}.

\begin{proposition}
Let $\mathcal{D}$ be a distribution on graphs with at most $n^*$ vertices, let 	$r^{(k)}_{\bW}$ be a \kmodel{}, let $r_{\bW}^{(k, \text{GNN})}$ be a \gnn{}, and let $h_{\bW}^{\text{GNN}}$ be the underlying  GNN graph representation used within the \gnn{} $r_{\bW}^{(k, \text{GNN})}$.
Assume that the GNN $h_{\bW}^{\text{GNN}}$ is \emph{$(\varepsilon/{n^* \choose k},k)$-universal}, then with probability $1-\varepsilon$ it holds that 
\begin{align*}
    r_{\bW}^{(k, \text{GNN})} \equiv r^{(k)}_{\bW}.
\end{align*}
\end{proposition}
\begin{proof}
By the union bound, we can upper bound the probability that at least one $k$-vertex subgraph is not uniquely represented by the GNN $h_{\bW}^{\text{GNN}}$ by 
\begin{equation*}
    \sum_{i=1}^{{n^* \choose k}} \frac{\varepsilon}{{n^* \choose k}} = \varepsilon.
\end{equation*}
Hence, we can lower bound the probability that this never happens by $1-\varepsilon$.
\end{proof}

\subsection{Proof of \Cref{thm:cycle}}

In the following, we proof \Cref{thm:cycle}. The following result, showing that the $1$-WL assigns unique representation to forest graphs, follows directly from~\cite{Arv+2015}.

\begin{lemma}[\citet{Arv+2015}]\label{lem:wltrees} 1-WL distinguishes any pair of non-isomorphic forests.
\end{lemma}

The following results shows that the degree list of a graph is reconstructable.

\begin{lemma}[\citet{taylor1990reconstructing}]\label{lem:deg} The degree list of an $n$-vertex graph is $(n-\ell)$-reconstructible if $$ n \geq (\ell -\log\ell + 1) \Big( \frac{e + e\log\ell + e + 1}{(\ell-1) \log \ell - 1  }\Big) + 1. $$
\end{lemma}

The following results shows that connectedness of a graph is reconstructable.

\begin{lemma}[\citet{spinoza2019reconstruction}]\label{lem:conn} Connectedness of an $n$-vertex graph is $(n-\ell)$-reconstructible if $$ \ell < (1+o(1)) \Big(  \frac{2\log n}{ \log \log n }  \Big)^{1/2}.$$.
\end{lemma}

Moreover, we will need the following observation.

\begin{observation}\label{obs:cycles} Every subgraph of a cycle graph is either a path graph or a collection of path graphs.
\end{observation}

We can now prove \Cref{thm:cycle}.

\begin{proof}[Proof of \Cref{thm:cycle}]

We start with a simple observation in \Cref{obs:cycles}. With that, we know from \Cref{lem:wltrees} that GNNs can assign unique representations to every cycle subgraph. Thus, it follows from \Cref{prop:kmodel} that \gnn{s} can learn $k$-reconstructible functions of cycle graphs. With that, if conditions \textit{i)} and \textit{ii)} hold, it follows from \Cref{lem:deg,lem:conn} that we can reconstruct the degree list and the connectedness of a cycle graph. Note that a cycle graph is a $2$-regular connected graph. That is, it is uniquely identified by its degree list and its connectedness. Thus, \gnn{s} can assign an unique representation to it if conditions \textit{i)} and \textit{ii)} hold.

\end{proof}

\subsection{Proof of \Cref{thm:csl}}

The following definition defines CSL graphs. 

\begin{definition}[Circular Skip Link (CSL) graphs~\citep{murphy2019relational}]\label{def:csl}
Let $R$ and $M$ be co-prime natural numbers such that $R < M - 1$. We denote by $\cG_{\text{skip}}(M,R)$ the undirected $4$-regular graph with vertices labeled as $0,1,...,M-1$ whose edges form a cycle and have skip links. More specifically, the edge set is defined by a cycle formed by $(i,i+1),(i+1,i)$ in $E$ for $i \in \{ 1,...,M-2\}$ and $(0,M-1),(M-1,0)$ together with skip links defined recursively by the sequence of edges $(s_i,s_{i+1}),(s_{i+1},s_{i}) \in E$ with $s_1=0, s_{i+1}= (s_i + R) \mod M $

\end{definition}

\begin{proof}[Proof of \Cref{thm:csl}]
Consider two non-isomorphic CSL graphs with the same number of vertices, which we can denote by $\cG_{\text{skip}}(M,R)$ and $\cG_{\text{skip}}(M,R')$ with $R \neq R'$ according to \Cref{def:csl}. First, note that every $n-1$-size subgraph (card) of a CSL graph is isomorphic to each other. Thus, for \gnn{s}---due to the equivalence in expressiveness between GNNs and 1-WL---it suffices to prove that 1-WL can distinguish between a card from $\cG_{\text{skip}}(M,R)$ and a card from $\cG_{\text{skip}}(M,R')$.

Now, let $\cG_{\text{skip}}^{-i}(M,R)$ and $\cG_{\text{skip}}^{-i}(M,R')$ be the two subgraphs we get by removing the vertex $i$ from $\cG_{\text{skip}}(M,R)$ and $\cG_{\text{skip}}(M,R')$ respectively. In each subgraph, $M-4$ vertices remain with degree 4. However, we can differentiate the two subgraphs by looking at the vertices which now have degree 3. In both subgraphs $i-1$ and $i+1$ will have degree 3. Moreover, in $\cG_{\text{skip}}^{-i}(M,R)$ the vertex $(i + R) \mod M$ will have degree 3, while $(i + R') \mod M$ in $\cG_{\text{skip}}^{-i}(M,R')$ will have degree 3. Since $(i + R) \mod M \neq (i + R') \mod M$, the distance from $i+1$ to $(i + R) \mod M$ and to $(i + R') \mod M$ is different in the two graphs. Thus, the $1$-WL will assign different colors to the vertices between $i+1$ and $(i + R) \mod M$ in a subgraph and between $i+1$ and $(i + R') \mod M$ in the other. The same argument applies to the distance from $j$, where $(j + R) \mod M \equiv i$  to $i-1$. Hence, the $1$-WL will assign different color histograms to the subgraphs and thus \gnn{s} can distinguish them.\qedhere
\end{proof}

\subsection{Proof of \Cref{prop:gnn-smallk}}

 We start by stating the following result.
\begin{lemma}[\citet{nydl1981finite}]\label{lem:spider} Spider graphs are not $\ceil{n/2}$-reconstructible.
\end{lemma}

\begin{proof}[Proof of \Cref{prop:gnn-smallk}]

First, it is clear that \gnn{s} $\preceq$ \kmodel{s}, thus it suffices to show  GNNs $\not \preceq$ \kmodel{s} for $k \leq \ceil{n/2} $. It follows from \Cref{lem:spider} and \Cref{prop:kmodel} that \kmodel{s} cannot assign unique representations to spider graphs if $k=\ceil{n/2}$. However, spider graphs are a family of trees, which are known to be assigned unique representations in 1-WL, see~\Cref{lem:wltrees}. Thus, due to the equivalence in expressiveness between GNNs and 1-WL, we know GNNs can assign unique representations to spider graphs. Thus, GNNs $\not \preceq$ \kmodel{s} for $k = \ceil{n/2} $. In fact, from \Cref{obs:kdeck}, we know GNNs $\not \preceq$ \kmodel{s} for $k \leq \ceil{n/2} $. \qedhere
\end{proof}

\subsection{Proof of \Cref{prop:2wl}}\label{supp:2wl}
In the following, we provide a pair of non-isomorphic graphs that the $2$-WL cannot distinguish, while a \gnn{} with $k:=n-2$ can.

\begin{proof}[Proof sketch of \Cref{prop:2wl}]

The $2$-WL cannot distinguish any pair of non-isomorphic, strongly-regular graphs with the same parameters~\cite{Gro+2020b}. For example, following~[A1], the $2$-WL cannot distinguish the line graph of $K_{4,4}$ (graph $G_1$) and Shrikande Graph (graph $G_2$), both strongly-regular graphs with parameters $(16, 6, 2, 2)$, which are non-isomorphic.\footnote{	\url{https://www.win.tue.nl/~aeb/graphs/srg/srgtab.html}
} Hence, any $2$-GNN architecture can also not distinguish them. The graph $G_1$ and $G_2$ are non-isomorphic since the neighborhood around each node in the graph either induces a cycle or a disjoint union of triangles, respectively. By a similar argument as in the proof of~\cref{thm:csl}, the $1$-WL can distinguish the graphs induced by the decks $\mathcal{D}_{2}(G_1)$ and $\mathcal{D}_{2}(G_2)$, which we verified by a computer experiment. Hence, there exists a $(n-2)$-Reconstruction $2$-GNN architecture that can distinguish the two graphs. 3
\end{proof}

\subsection{WL reconstruction conjecture}

There exists a wide variety of graphs identifiable by 1-WL---and thus by GNNs---with one or a few $(n-1)$-size subgraphs not identifiable by 1-WL. A simple example would be adding a special vertex to a cycle with 5 vertices. In this new 6-vertex graph we connect the special vertex to every other vertex in the 5-cycle. Further, we add another 5-cycle as a different component. This new 11-vertex graph is identifiable by 1-WL from \citet[Theorem 17]{Kie+2015}.\footnote{The flip of this graph is a bouquet forest with two 5-cycles (of distinct colors) and an isolated extra vertex} However, the 10-vertex subgraph we get by removing the special vertex is a regular graph, notably not identifiable by 1-WL.

As we saw, getting one or even a few subgraphs that GNNs cannot distinguish from a distinguishable original graph is not a complex task. However, in order to understand whether GNNs are not less powerful than \gnn{s} we need to find a counter example where 1-WL cannot distinguish the entire multiset of $k$-vertex subgraphs. In this work, we were not able to find such example for large enough $k$, \textit{i.e.} $k \approx n$. Thus, we next state what we name the WL reconstruction conjecture.

\begin{conjecture}\label{conj:wl}
    For $\cS^{(k)}$ as the set of all $k$-size subsets of $V(G)$, let $G_k = \sum_{S \in \cS^{(k)}} G[S] $ be the disjoint union of all $k$-vertex subgraphs of $G$. Then, there exists some $k \in [n]$ such that if $G$ is uniquely identifiable by 1-WL, $G_k$ is.
\end{conjecture}

If \Cref{conj:wl} holds, GNNs $\prec$ \gnn{s}.

\subsection{Proof of \Cref{thm:var}}

We start by extending the Invariance Lemma~\citep[Lemma 4.1]{Chen2019inv} to our context in \Cref{lemma}.

\begin{lemma}\label{lemma} Let $\mu$ be an arbitrary $\delta$-hereditary property and $P_{\cD}$ as in \Cref{thm:var}. now, let $ \mu_k(G) := \E_{ S \sim \text{Unif}(\cS^{(k)}) }[ \mu( G[S] ) ] = \frac{1}{|\cS^{(k)}|} \sum_{S \in \cS^{(k)}} \mu(G[S]) $. Then:

\begin{itemize}
    \item[i)] By inspection, for any $G \in \cG$ with $|V(G)| \geq \delta + \ell $, $ \mu_k(G) = \E[ \mu( H ) \colon H \in \cG_k(G)  ] $ where $ \cG_k(G) := \{ G[S] \colon S \in \cS^{(k)} \}$.
    \item[ii)] By the law of total expectation, $ \E_{P_\cD} [ \mu(G) ] = \E_{P_\cD} [ \mu_k(G) ] $.
    \item[iii)] Note that the covariance matrices of $\mu(G)$ and any of its subgraphs $\mu(G[S])$ are equal, i.e.,  $ \text{Cov}_{P_\cD} \mu(G) = \text{Cov}_{P_\cD, \text{Unif}(\cS^{(k)})} \mu( G[S] ) $. Thus, by the law of total covariance,
        $$ \text{Cov}_{P_\cD}[\mu(G)]  =  \text{Cov}_{P_\cD}[\mu_k(G)] +  \E_{P_\cD}[\text{Cov}_{\text{Unif}(\cS^{(k)})}[\mu(G)]] .$$
\end{itemize}
\end{lemma}

\begin{proof}[Proof of \Cref{thm:var}] Let us first define three risk estimators:

\begin{itemize}
    \item[a)] GNN estimator:
    $$\widehat{\cR}_{\text{GNN}}(\cD^{(\text{tr})};\bW_2, \bW_3) := \frac{1}{N^{\text{tr}}}  \sum_{i=1}^{N^{\text{tr}}}  l \Big(\rho_{\bW_{1}} \Big( \phi_{\bW_{2}} \Big( h_{\bW_3}^{\text{GNN}}(G[S])\Big) \Big) , y_i \Big)$$
    \item[b)] Data augmentation estimator:
     $$\widehat{\cR}_{\circ}(\cD^{(\text{tr})};\bW_2, \bW_3) := \frac{1}{N^{\text{tr}}}  \sum_{i=1}^{N^{\text{tr}}} 1/|\cS^{(k)}| \sum_{S \in \cS^{(k)}} l\Big( \rho_{\bW_{1}}\Big( \phi_{\bW_{2}} \Big( h_{\bW_3}^{\text{GNN}}(G[S])\Big) \Big), y_i\Big)$$
    \item[c)] \gnn{} estimator:
     $$\widehat{\cR}_k(\cD^{(\text{tr})};\bW_1, \bW_2, \bW_3) := \frac{1}{N^{\text{tr}}}  \sum_{i=1}^{N^{\text{tr}}} l\Big( \rho_{\bW_{1}}\Big( 1/|\cS^{(k)}| \sum_{S \in \cS^{(k)}} \phi_{\bW_{2}} \Big( h_{\bW_3}^{\text{GNN}}(G[S]) \Big) , y_i\Big)$$
\end{itemize}

Now, we leverage \Cref{lemma}. By mapping $\phi_{\bW_{2}} \Big( h_{\bW_3}^{\text{GNN}}(G[S])\Big)$ to $\mu$ and $1/|\cS^{(k)}| \sum_{S \in \cS^{(k)}} l\big( \phi_{\bW_{2}} \Big( h_{\bW_3}^{\text{GNN}}(G[S])\Big) , y_i\big)$ to $\mu_k$, from iii) we get that

$$ \text{Var}[ \widehat{\cR}_{\circ}(\cD^{(\text{tr})};\bW_2, \bW_3)  ] \leq \text{Var}[ \widehat{\cR}_{\text{GNN}}(\cD^{(\text{tr})};\bW_2, \bW_3)  ] $$

Since $l \circ \rho_{\bW_1}$ is convex in the first argument, we apply Jensen's inequality and arrive at

$$  \text{Var}[ \widehat{\cR}_{k}(\cD^{(\text{tr})};\bW_1, \bW_2, \bW_3)  ]  \leq \text{Var}[ \widehat{\cR}_{\circ}(\cD^{(\text{tr})};\bW_2, \bW_3)  ] \leq \text{Var}[ \widehat{\cR}_{\text{GNN}}(\cD^{(\text{tr})};\bW_2, \bW_3)  ] ,$$

as we wanted to show.\qedhere

\end{proof}

\section{Details on Experiments and Architectures}\label{supp:exp}
In the following, we give details on the experiments. 

\subsection{1-WL test on real-world datasets.} To show how the expressiveness of GNNs is not an obstacle in real-world tasks, we tested if the $1$-WL can distinguish each pair of non-isomorphic graphs in every dataset used in \Cref{sec:exp}.  We go further and ignore vertex and edge features in the test, showing how only the graphs' topology is enough for GNNs to distinguish $\approx 100\%$ of the graphs in each dataset. Results are shown in \Cref{tab:wlds}.
\begin{table}[h]
    \centering
    \begin{tabular}{lc}
     \textbf{Dataset}    & \textbf{\% of dist. non-iso. graph pairs} \\
          \toprule
     \textsc{zinc} & 100.00\,\% \\
     \textsc{alchemy} & >99.99\,\% \\
     \textsc{ogbg-moltox21} & >99.99\,\% \\
     \textsc{ogbg-moltoxcast} & >99.99\,\%  \\
     \textsc{ogbg-molfreesolv} &100.00\,\% \\
     \textsc{ogbg-molesol} &100.00\,\% \\
     \textsc{ogbg-mollipo} &100.00\,\% \\
    \textsc{ogbg-molhiv} & >99.99\,\%\\
    \textsc{ogbg-molpcba} & >99.99\,\%\\
    \bottomrule
    \end{tabular}
    \caption{Percentage of distinguished non-isomorphic graph pairs by $1$-WL over the used benchmark datasets.}
    \label{tab:wlds}
\end{table}

\subsection{Architectures.}

In the following, we outline details on the used GNN architectures. 

\xhdr{GIN.} Below we specify the architecture together with its \gnn{} version for each dataset.

\textsc{zinc}: We used the exact same architecture as used in \citet{Morris2020b}. Its reconstruction versions used a Deep Sets function with mean pooling with three hidden layers before the pooling and two after it. All hidden layers are of the same size as the GNN layers.

\textsc{alchemy}:  We used the exact same architecture as used in \citet{Morris2020b}. Its reconstruction versions used a Deep Sets function with mean pooling with one hidden layer before the pooling and three after it. All hidden layers are of the same size as the GNN layers.

\textsc{ogbg-molhiv}:  We used the exact same architecture as used in \citet{hu2020ogb}. Its reconstruction versions used a Deep Sets function with mean pooling with no hidden layer before the pooling and two after it. Additionally, a dropout layer before the output layer. All hidden layers are of the same size as the GNN layers.

\textsc{ogbg-moltox21}: Same as \textsc{ogbg-molhiv}.

\textsc{ogbg-moltoxcast}:  We used the exact same architecture as used in \citet{hu2020ogb}. Its reconstruction versions used a Deep Sets function with mean pooling with no hidden layer before or after it.

\textsc{ogbg-molfreesolv}: We used the exact same architecture as used in \citet{hu2020ogb}, with the exception of using their jumping knowledge layer, which yielded better validation and test results. Its reconstruction versions used a Deep Sets function with mean pooling with no hidden layer before or after it.

\textsc{ogbg-molesol}: Same as \textsc{ogbg-molfreesolv}.

\textsc{ogbg-mollipo}: We used the exact same architecture as used in \citet{hu2020ogb}, with the exception of using their jumping knowledge layer, which yielded better validation and test results. Its reconstruction versions used a Deep Sets function with mean pooling with no hidden layer before and three after it. All hidden layers are of the same size as the GNN layers.

\textsc{ogbg-molpcba}: We used the exact same architecture as used in \citet{hu2020ogb}. Its reconstruction versions used a Deep Sets function with mean pooling with one hidden layer before the pooling and no after it. All hidden layers are of the same size as the GNN layers.

\textsc{csl}: We used the exact same architecture as used in \citet{Dwi+2020}. Its reconstruction versions used a Deep Sets function with mean pooling with no hidden layer before the pooling and one after it. All hidden layers are of the same size as the GNN layers (110 neurons).

\textsc{multitask}: Same as \textsc{csl}, but with hidden layers of size 300.

\textsc{4,6,8 cycles}: We used the same GIN architecture from \textsc{csl}, with the difference of using one hidden layer and two after the aggregation in the Deep Sets architecture. Additionally, a dropout layer before the output layer. All hidden layers are of the same size as the GNN layers, \textit{i.e.} 300.

\xhdr{GCN.} We used the exact same architectures from GIN for each dataset with the only change being the convolution (aggregation) layer, here we used the GCN layer from \citet{Kip+2017} instead of GIN.

\xhdr{PNA.} Below we specify the architecture together with its \gnn{} version for each dataset.

\textsc{zinc}: We used the exact same architecture from \citet{Cor+2020}. Its reconstruction versions used a Deep Sets function with mean pooling with one hidden layer before the pooling and three after it. All hidden layers have 25 hidden units.

\textsc{alchemy}: We used the exact same architecture used for \textsc{zinc} in \citet{Cor+2020}. Its reconstruction versions used a Deep Sets function with mean pooling with one hidden layer before the pooling and three after it. All hidden layers have 25 hidden units.

\textsc{ogbg-molhiv}: We used the exact same architecture as used in \citet{Cor+2020}. Its reconstruction versions used a Deep Sets function with mean pooling with one hidden layer before the pooling and three after it. In the original PNA model the hidden layers after vertex pooling are of sizes 70, 35 and 17. We replaced them so all have 70 hidden units and put the layers after the subgraph pooling with sizes 70, 35 and 17.

\textsc{ogbg-moltox21}: Same as \textsc{ogbg-molhiv}.

\textsc{ogbg-moltoxcast}: Same as \textsc{ogbg-molhiv}.

\textsc{ogbg-molfreesolv}: We used the exact same architecture used for \textsc{ogbg-molhiv} in \citet{Cor+2020}. Its reconstruction versions used a Deep Sets function with mean pooling with no hidden layer before or after the pooling.

\textsc{ogbg-molesol}: Same as \textsc{ogbg-molfreesolv}.

\textsc{ogbg-mollipo}: Same as \textsc{ogbg-molhiv}.

\textsc{ogbg-molpcba}: We used the exact same architecture used for \textsc{ogbg-molhiv} in \citet{Cor+2020}. Its reconstruction versions used a Deep Sets function with mean pooling with no hidden layer before and two after the pooling.  Additionally, a dropout layer before the output layer. All hidden units are of size 510, with exception of the two hidden layers in Deep Sets that had 255 and 127 neurons.

\textsc{csl}: Same as \textsc{alchemy}.

\textsc{multitask}: We used the exact same architecture from in \citet{Cor+2020}. Its reconstruction versions used a Deep Sets function with mean pooling with no hidden layer before and two after the pooling. All hidden units are of size 16.

\textsc{4,6,8 cycles}:  We used the exact same architecture from in \citet{Cor+2020}, with the difference of a sum pooling instead of a Set2Set~\citep{vinyals2015order} pooling for the readout function in PNA. Its reconstruction versions used a Deep Sets function with mean pooling with no hidden layer before and two after the pooling.  Additionally, a dropout layer before the output layer. All hidden units are of size 16.

\subsection{Experimental setup.}

As mentioned in \Cref{sec:exp}, we retain training procedures and evaluation metrics from the original GNN works~\citep{Dwi+2020,Morris2020b,hu2020ogb}. We highlight how \textsc{csl} is the only dataset with a $k$-fold cross validation (with $k=5$) as originally proposed in \citep{Dwi+2020}. In \Cref{tab:sampling}, we highlight the number of subgraph samples used for training and testing each \gnn{} archictecture for every dataset, our only new hyperparameter introduced. Note that for test, for what is not specified in \Cref{tab:sampling} we use 200 samples or compute exactly (if number of subgraphs $\leq$ 200) for all architectures in all datasets..

\xhdr{Implementation.} All models were implemented in PyTorch Geometric\citep{Fey+2019} using NVidia GeForce 1080 Ti GPUs.

\begin{table*}
    \centering
    \resizebox{0.6\textwidth}{!}{\renewcommand{\arraystretch}{1.0}
    \begin{tabular}{lccccc}
      & & \multicolumn{4}{c}{\bf \# of samples used} \\
      & \bf GNN &  \multicolumn{4}{c}{\bf in \gnn{}} \\
    \bf Dataset & \bf architecture   & $n-1$  & $n-2$  & $n-3$ &  $\ceil{n/2}$   \\
          \toprule
     \textsc{zinc} & \textbf{GIN} & Exact  & 10 &  10 & 10 \\
     \textsc{zinc} & \textbf{GCN} &  Exact & 10  & 10  & 10 \\
     \textsc{zinc} & \textbf{PNA} & Exact  & 10 & 10 & 10 \\
     \textsc{alchemy} & \textbf{GIN}   & Exact & 30 & 30 & 30 \\
     \textsc{alchemy} & \textbf{GCN}  &  Exact  & 30 & 30 & 30 \\
     \textsc{alchemy} & \textbf{PNA}  & Exact  & 30 & 30 & 30 \\
     \textsc{ogbg-moltox21} & \textbf{GIN}  & 5  & 5 & 5 & 5 \\
     \textsc{ogbg-moltox21} & \textbf{GCN}  & 5  & 5 & 5 & 5 \\
     \textsc{ogbg-moltox21} & \textbf{PNA}  & 5  & 5 & 5 & 5 \\
     \textsc{ogbg-moltoxcast} & \textbf{GIN}  & 5  & 5 & 5 & 5 \\
     \textsc{ogbg-moltoxcast} & \textbf{GCN}  & 5  & 5 & 5 & 5 \\
     \textsc{ogbg-moltoxcast} & \textbf{PNA}  & 5  & 5 & 5 & 5 \\
     \textsc{ogbg-molfreesolv} & \textbf{GIN}  & Exact  & Exact & Exact & Exact \\
     \textsc{ogbg-molfreesolv} & \textbf{GCN}  & Exact  & Exact & Exact & Exact \\
     \textsc{ogbg-molfreesolv} & \textbf{PNA}  & Exact  & 20 & 20 & 20 \\
     \textsc{ogbg-molesol} & \textbf{GIN}   & Exact  & Exact & Exact & Exact  \\
     \textsc{ogbg-molesol} & \textbf{GCN}   & Exact  &  Exact & Exact & Exact  \\
     \textsc{ogbg-molesol} & \textbf{PNA}   & Exact  & 20 & 20 & 20  \\
     \textsc{ogbg-mollipo} & \textbf{GIN}   & Exact  & Exact & Exact & Exact \\
     \textsc{ogbg-mollipo} & \textbf{GCN}  & Exact  & Exact & Exact & Exact \\
      \textsc{ogbg-mollipo} & \textbf{PNA}  & 20  & 20 & 20 & 20 \\
    \textsc{ogbg-molhiv} & \textbf{GIN}  & 5  & 5 & 5 & 5 \\
     \textsc{ogbg-molhiv} & \textbf{GCN} & 5  & 5 & 5 & 5  \\
      \textsc{ogbg-molhiv} & \textbf{PNA}  & 5  & 5 & 5  & 5 \\
      
      \textsc{ogbg-molpcba} & \textbf{GIN}  & 3/5$^+$  & 3/5$^+$  & 3/5$^+$   & 3/5$^+$ \\
    \textsc{ogbg-molpcba} & \textbf{GIN}  & 3/5$^+$  & 3/5$^+$  & 3/5$^+$   & 3/5$^+$ \\
    \textsc{ogbg-molpcba} & \textbf{PNA}  & 1/3$^+$  & 1/3$^+$ & 1/3$^+$  & 1/3$^+$ \\
      
      \textsc{csl} & \textbf{GIN}  & Exact  & 20 & 20 & 20  \\
     \textsc{csl} & \textbf{GCN} & Exact  & 20 & 20 & 20 \\
      \textsc{csl} & \textbf{PNA}  & Exact  & 20 & 20 & 20  \\
      
      \textsc{multitask} & \textbf{GIN}  &  25/20$^+$ & 25/20$^+$ & 25/20$^+$ & 25/20$^+$  \\
     \textsc{multitask} & \textbf{GCN} & 25/20$^+$  & 25/20$^+$ & 25/20$^+$ & 25/20$^+$ \\
      \textsc{multitask} & \textbf{PNA}  & 15/10$^+$  & 15/10$^+$ & 15/10$^+$  & 15/10$^+$  \\
      
       \textsc{4,6,8 cycles} & \textbf{GIN}  & 10/10$^+$ & 10/10$^+$ & 10/10$^+$ & 10/10$^+$ \\
     \textsc{4,6,8 cycles} & \textbf{GCN} & 10/10$^+$ & 10/10$^+$ & 10/10$^+$ & 10/10$^+$ \\
      \textsc{4,6,8 cycles} & \textbf{PNA} & 10/10$^+$ & 10/10$^+$ & 10/10$^+$ & 10/10$^+$ \\

    \bottomrule
    \end{tabular}
    }
    \caption{The number of subgraph samples used in training in each \gnn{} for every dataset. \\$+$: train/test (validation same as test).}
    \label{tab:sampling}
\end{table*}

\section{Datasets}\label{supp:datasets}

In \Cref{tab:data}, we show some basic statistics from the datasets used in \Cref{sec:exp}.

\begin{table*}
    \centering
    \resizebox{0.9\textwidth}{!}{\renewcommand{\arraystretch}{1.0}
    \begin{tabular}{lcccc}
     \textbf{Dataset}    & \textbf{ \# of graphs } &  \textbf{ \# of classes/targets } &  \textbf{ Average \# of vertices } &  \textbf{ Average \# of edges }  \\
          \toprule
     \textsc{zinc}$^1$ & 249\,456  & 1 & 23.1 & 24.9 \\
     \textsc{alchemy}$^2$ & 202\,579 & 12 & 10.1 & 10.4 \\
     \textsc{ogbg-moltox21} & 7\,831 & 12 & 18.6 & 19.3 \\
     \textsc{ogbg-moltoxcast} & 8\,576 & 617 & 18.8 & 19.3 \\
     \textsc{ogbg-molfreesolv} & 642 & 1 & 8.7 & 8.4 \\
     \textsc{ogbg-molesol} & 1\,128 & 1 & 13.3 & 13.7 \\
     \textsc{ogbg-mollipo} & 4\,200 & 1 & 27.0 & 29.5 \\
    \textsc{ogbg-molhiv} & 41\,127  & 2 & 25.5 & 27.5 \\
    \textsc{ogbg-molpcba} & 437\,929  & 128 & 26.0 & 28.1 \\
    \textsc{csl} & 150  & 10 & 41.0 & 82.0 \\
    \textsc{multitask} & 7\,040 & 3 & 18.81 & 47.58 \\
    \textsc{4 cycles}$^3$ & 20\,000 & 2 & 36.0 & 30.85 \\
    \textsc{6 cycles}$^4$ & 20\,000 & 2 & 48.96 & 43.92 \\
    \textsc{8 cycles}$^5$ & 20\,000 & 2 & 61.96 & 56.94 \\
    \bottomrule
    \end{tabular}
    }
    \caption{Dataset statistics. $^1$---We used the 10k subset from \citet{Dwi+2020}.  $^2$---We used the 10k subset from \citet{Morris2020b}. We generated the datasets from \citet{Vig+2020} using an average graph size of $^3$ 36, $^4$ 56, $^5$ 72}
    \label{tab:data}
\end{table*}

\section{Additional results}\label{addres}

The attentive reader might wonder how sampling subgraphs impacts the training of the models. More precisely, how does sampling subgraphs affect the convergence and the accuracy of models? To provide insights in this matter, we show results for $(n-1)$-Reconstruction GIN in the \textsc{alchemy} dataset. \Cref{fig:test} shows the average number of epochs taken to converge and the training loss at convergence. We note that the training converges faster to a larger loss for a very small sample size, e.g., 1 and 3. Sample sizes of 8 and 15 are already sufficient to converge to approximately the training loss with the exact model. Note that 15 is already the average graph size in the dataset. Thus it has a very similar behavior as the exact model. For a sample size of 8, the model takes 21\% more epochs to converge to the same training loss as the exact. Note, however, that using subgraph samples uses a fixed amount of GPU memory independently of the maximum graph size in the dataset.

\begin{figure}
\centering
\begin{subfigure}{.5\textwidth}
  \centering
  \includegraphics[scale=0.3]{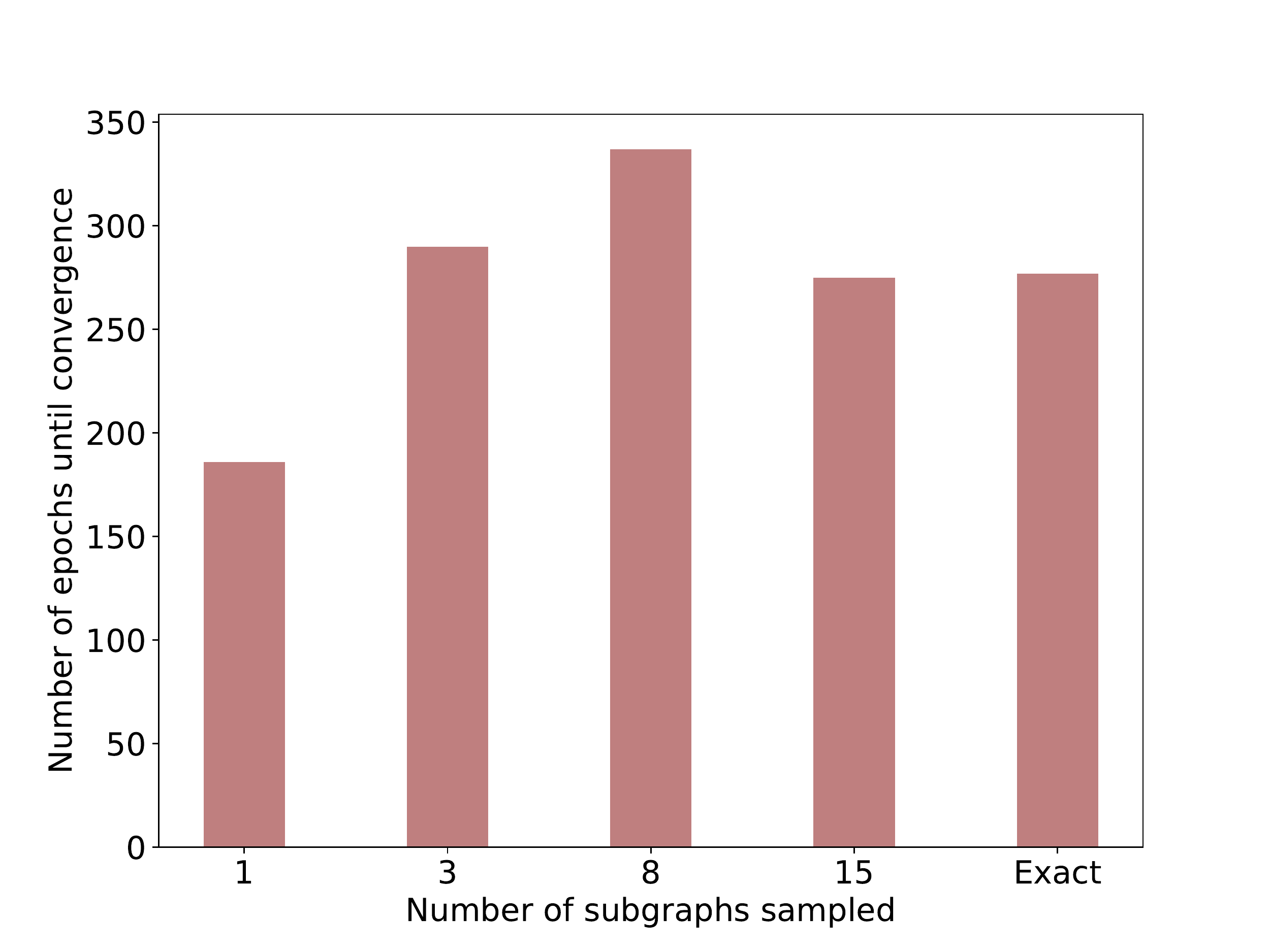}
  \caption{Number of epochs until convergence under different number of subgraphs sampled in training.}
  \label{fig:epoch}
\end{subfigure}\begin{subfigure}{.5\textwidth}
  \centering
  \includegraphics[scale=0.3]{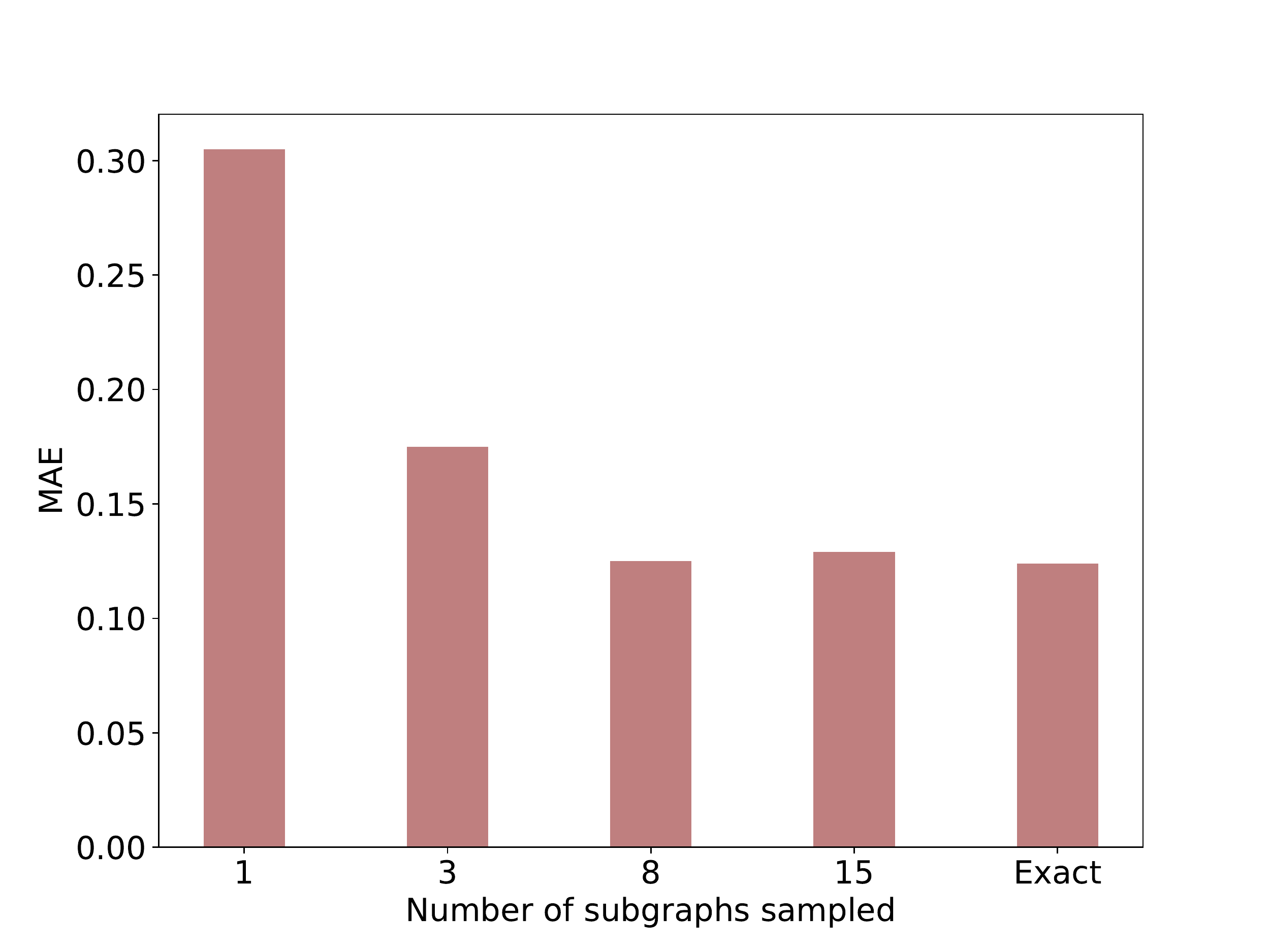}
  \caption{Mean absolute error under different number of subgraphs sampled in training. }
  \label{fig:mae}
\end{subfigure}
\caption{The impact of sampling subgraphs in $(n-1)$-Reconstruction GIN on the \textsc{alchemy} dataset}
\label{fig:test}
\end{figure}

 \begin{table}[!htbp]

 \centering
 \resizebox{0.49\textwidth}{!}{\renewcommand{\arraystretch}{0.95}
   \begin{tabular}[t]{@{}m{0.05em}lccc@{}}
 \toprule
              && \textsc{ogbg-molhiv} & \textsc{zinc} & \textsc{alchemy}
              \\
              && (ROC-AUC \%) $\uparrow$  & (MAE) $\downarrow$ &  (MAE) $\downarrow$
              \\
 \cmidrule{2-5}
 & \textbf{GIN}    &  75.58 \footnotesize{ $\pm$ 1.40 }    &    0.278 \footnotesize{$\pm$ 0.022}     &    0.185 \footnotesize{$\pm$ 0.022}                    \\

 \parbox[t]{2mm}{\multirow{4}{*}{\rotatebox[origin=c]{90}{\scriptsize{Reconstruction }}}}
 & \footnotesize{$(n-1)$}     & \colorbox{pastelgreen}{ 76.32 \footnotesize{ $\pm$ 1.40 } }     &  \colorbox{pastelgreen}{ 0.209 \footnotesize{$\pm$ 0.009} }    &    \colorbox{pastelgreen}{   0.160  \footnotesize{$\pm$ 0.003} }          \\

 & \footnotesize{$(n-2)$}    &   \colorbox{pastelgreen}{ 77.53 \footnotesize{ $\pm$ 1.59 } }   &   0.324 \footnotesize{$\pm$ 0.048}   &  \colorbox{pastelgreen}{   0.153  \footnotesize{$\pm$ 0.003} }        \\

 & \footnotesize{$(n-3)$}        &    \colorbox{pastelgreen}{ 75.82 \footnotesize{ $\pm$ 1.65 } }  &     0.329 \footnotesize{$\pm$ 0.049} &   \colorbox{pastelgreen}{   0.167  \footnotesize{$\pm$ 0.006} }    \\

 & \footnotesize{$\ceil{n/2}$} & 68.43 \footnotesize{ $\pm$ 1.23 }      &        0.548   \footnotesize{$\pm$ 0.006}    &    0.238 \footnotesize{$\pm$ 0.011}    \\

 \cmidrule{2-5}

 & \textbf{GCN}    &   76.06 \footnotesize{ $\pm$ 0.97 }     &    0.306 \footnotesize{$\pm$ 0.023}       &   0.189  \footnotesize{$\pm$ 0.003}                  \\

 \parbox[t]{2mm}{\multirow{4}{*}{\rotatebox[origin=c]{90}{\scriptsize{Reconstruction }}}}
 & \footnotesize{$(n-1)$}     &  \colorbox{pastelgreen}{  76.83  \footnotesize{ $\pm$ 1.88 }  }  &    \colorbox{pastelgreen}{   0.248 \footnotesize{$\pm$ 0.011} }    &   \colorbox{pastelgreen}{   0.162  \footnotesize{$\pm$ 0.002} }               \\

 & \footnotesize{$(n-2)$}    &  \colorbox{pastelgreen}{  76.13  \footnotesize{ $\pm$ 1.18 } }    &     0.340   \footnotesize{$\pm$ 0.025}    &   \colorbox{pastelgreen}{   0.157  \footnotesize{$\pm$ 0.004} }         \\

 & \footnotesize{$(n-3)$}        &   76.00 \footnotesize{ $\pm$ 3.30}    &     0.361  \footnotesize{$\pm$ 0.015}&   \colorbox{pastelgreen}{ 0.161 \footnotesize{ $\pm$ 0.004}}  \\

 & \footnotesize{$\ceil{n/2}$} &   72.16 \footnotesize{ $\pm$ 1.96 }      &   0.544   \footnotesize{$\pm$ 0.006}    &         0.236 \footnotesize{$\pm$ 0.013}   \\

 \cmidrule{2-5}

 & \textbf{PNA}    &   \textbf{79.05} \footnotesize{ $\pm$ 1.32 }    &     0.188 \footnotesize{$\pm$ 0.004}     &    0.176   \footnotesize{$\pm$ 0.011}                       \\

 \parbox[t]{2mm}{\multirow{4}{*}{\rotatebox[origin=c]{90}{\scriptsize{Reconstruction }}}}
 & \footnotesize{$(n-1)$}     &  77.88 \footnotesize{ $\pm$ 1.13 }   &      \colorbox{pastelgreen}{ 0.170 \footnotesize{$\pm$ 0.006} }    &   \colorbox{pastelgreen}{   0.125 \footnotesize{$\pm$ 0.001} }                                \\

 & \footnotesize{$(n-2)$}    &   78.49 \footnotesize{ $\pm$ 1.33 }   &  0.197 \footnotesize{$\pm$ 0.007}  &  \colorbox{pastelgreen}{   0.128 \footnotesize{$\pm$ 0.002} }                   \\

 & \footnotesize{$(n-3)$}        &    78.85  \footnotesize{ $\pm$ 0.48 } &       0.212 \footnotesize{$\pm$ 0.212}   &  \colorbox{pastelgreen}{   0.152 \footnotesize{$\pm$ 0.006} } \\

 & \footnotesize{$\ceil{n/2}$} &   76.48 \footnotesize{$\pm$ 0.35}  &  0.582 \footnotesize{$\pm$ 0.018}  &   0.243 \footnotesize{$\pm$ 0.005}\\

 \cmidrule{2-5}

 & \textbf{LRP} &  77.19  \footnotesize{$\pm$ 1.40}    &    0.223   \footnotesize{$\pm$ 0.001}     & --- \\
 & \textbf{GSN} &  77.99   \footnotesize{$\pm$ 1.00}    &    \textbf{0.108} \footnotesize{$\pm$ 0.001}   & --- \\
 & \textbf{$\delta$-2-LGNN} &  --- & 0.306  \footnotesize{$\pm$ 0.044}  &    \textbf{0.122}  \footnotesize{$\pm$ 0.003}     \\
 & \textbf{SMP} &   --- &     $0.138^\dagger$~~~~~~~~~~~    &      ---   \\

 \bottomrule
 \end{tabular}
 }
   \caption{Further results. We highlight in green \gnn{s} that boost the original GNN architecture. $\dagger$: Standard deviation not reported in original work.}\label{tab:2_appendix}
 \end{table}
 
 \section{Result Analysis}
 
 \subsection{Graph Property Results}\label{appx:result_properties}
 
 Such results are a consequence of the benefits of \gnn{s} highlighted in \Cref{thm:cycle}. That is, when we remove vertices we transform cycles into paths and trees, graphs easily recognizable by GNNs. By reconstructing the original cycle from its subgraphs, we are able to solve the tasks. Moreover, reconstruction is also able to make GNNs solve the multitask of determining a graph's connectivity, diameter and spectral radius. When we remove a vertex (or a connected subgraph) from a graph, we expect to change only the representation of vertices in its connected component, thus, by inspecting unchanged representations we can determine connectivity and solve the task. The help of reconstruction in solving graph diameter and spectral radius is a bit less direct. The first is related to shortest path lengths and the latter to the number of walks of size $n$ on the graph. When we remove vertices, we alter vertex representations. The way these representations are changed are affected by both shortest path lengths and number of walks. Thus, by measuring the change in representations, reconstruction can better capture these metrics. Finally, we observe that opposed to Ring-GNN, PPGN and Positional GIN, \gnn{s} is the only method consistently solving the tasks while maintaining the most important feature of graph representations, i.e., invariance to vertex permutations.

\end{document}